\documentclass[11pt]{article}
\usepackage{graphicx}
\usepackage{tabularx}
\usepackage{url}
\usepackage{amsthm}
\usepackage{amsmath}
\usepackage{amsfonts}
\usepackage{amssymb}

\theoremstyle{plain}
\newtheorem{theorem}{Theorem}
\newtheorem{lemma}{Lemma}

\theoremstyle{definition}

\theoremstyle{remark}

\date{}

\begin{document}

\title{Exact learning for infinite families of concepts}
\author{Mikhail Moshkov\thanks{Computer, Electrical and Mathematical Sciences and Engineering Division,
King Abdullah University of Science and Technology (KAUST),
Thuwal 23955-6900, Saudi Arabia. Email: mikhail.moshkov@kaust.edu.sa.
}}
\maketitle

\begin{abstract}
In this paper, based on results of exact learning, test theory, and rough set theory, we study
arbitrary infinite families of concepts each of which consists of an
infinite set of elements and an infinite set of subsets of this set called concepts.
We consider the notion of a problem over a family of concepts that is described
by a finite number of elements: for a given concept, we should recognize which of
the elements under consideration belong to this concept. As algorithms for problem solving,
we consider decision trees of
five types: (i) using membership queries, (ii) using equivalence queries,  (iii)
using both membership and equivalence queries, (iv) using proper equivalence queries,
and (v) using both membership and proper equivalence queries. As time complexity, we
study the depth of decision trees. In the worst case, with the growth of the number of
elements in the problem description, the minimum depth of decision trees of the first
type either grows as a  logarithm or linearly, and the minimum depth of decision trees
of each of the other types either is bounded from above by a constant or grows as a
logarithm, or linearly. The obtained results allow us to distinguish seven complexity
classes of infinite families of concepts.
\end{abstract}


\section{Introduction \label{S1}}

Decision trees are widely used in many areas of computer science. They are studied, in particular, in exact learning initiated by Angluin \cite{Angluin88}, test theory 
initiated by Chegis and Yablonskii \cite{Chegis58}, and rough set theory
initiated by Pawlak \cite{Pawlak82}. In some sense, these theories deal with dual objects: for example,
attributes from test theory and rough set theory correspond to membership queries from exact
learning. In contrast to test theory and rough set theory, in exact learning besides membership
queries, equivalence queries are also considered.

Heged\"{u}s in \cite{Hegedus95} generalized some bounds from \cite{Moshkov83}
obtained in the framework of test theory to the case of exact learning with
membership and equivalence queries. Similar results were obtained
independently and in the other way by Hellerstein et al. \cite{Hellerstein96}%
. In  \cite{Moshkov01}, we moved in the opposite direction: we added to the model
considered in test theory and rough set theory the notion of a hypothesis that allowed us to use
an analog of equivalence queries, and studied decision trees using various
combinations of attributes, hypotheses, and proper hypotheses (an analog of
proper equivalence queries). The paper \cite{Moshkov01} did not contain proofs. The proofs and
some new results were considered in \cite{MTCS,MDAM}. The aim of the present paper
is to translate the results obtained in \cite{Moshkov01,MTCS,MDAM}  into exact learning.

In \cite{Moshkov01,MTCS,MDAM}, based on the results of exact learning \cite%
{Angluin88,Angluin04,Littlestone88,Maass92} and test theory and rough set theory \cite%
{Moshkov83,Moshkov89,Moshkov05}, we investigated infinite binary information
systems each of which consists of an infinite set of elements $A$ and an
infinite set $F$ of functions (attributes) from $A$ to $\{0,1\}$. We defined
the notion of a testing problem described by a finite number of attributes $%
f_{1},\ldots ,f_{n}$ from $F$: for a given element $a\in A$, we should
recognize the tuple $(f_{1}(a),\ldots ,f_{n}(a))$. To this end, we can use
decision trees based on two types of queries. We can ask about the value of
an attribute $f_{i}\in \{f_{1},\ldots ,f_{n}\}$. We will obtain an answer of
the kind $f_{i}(a)=\delta $, where $\delta \in \{0,1\}$. We can also ask if a
hypothesis $f_{1}(a)=\delta _{1},\ldots ,f_{n}(a)=\delta _{n}$ is true, where
$\delta _{1},\ldots ,\delta _{n}\in \{0,1\}$. Either this hypothesis will be
confirmed or we will obtain a counterexample $f_{i}(a)=\lnot \delta _{i}$,
which is chosen nondeterministically. The considered hypothesis is called
proper if there exists an element $b\in A$ such that $f_{1}(b)=\delta
_{1},\ldots ,f_{n}(b)=\delta _{n}$. As time complexity of a decision tree,
we considered its depth, which is equal to the maximum number of queries in
a path from the root to a terminal node of the tree.

For an arbitrary infinite binary information
system, we studied five functions of Shannon type, which characterize the dependence
in the worst case of the minimum depth of a decision tree solving a testing
problem on the number of attributes in the problem description. The
considered five functions correspond to the following five types of decision trees:

\begin{enumerate}
\item Only attributes are used in decision trees.

\item Only hypotheses are used in decision trees.

\item Both attributes and hypotheses are used in decision trees.

\item Only proper hypotheses are used in decision trees.

\item Both attributes and proper hypotheses are used in decision trees.
\end{enumerate}

We proved that the first function has two possible types of behavior:
logarithmic and linear. Each of the remaining functions has three possible
types of behavior: constant, logarithmic, and linear. We also studied joint
behavior of all five functions and described seven complexity classes of
infinite binary information systems.

An exact learning problem is defined by a domain $X$ and a concept class $B$
\cite{Angluin04}. The domain $X$ is a nonempty finite set. A concept is any
subset of $X$, and a concept class $B$ is a nonempty set of concepts. For a
given (but hidden) concept $c\in B$, we should recognize it using two types
of queries. A membership query: for an element $x\in X$, we ask if $x\in c$.
The answer is either $x\in c$ or $x\notin c$. An equivalence query: for a
concept $c^{\prime }\subseteq X$, we ask if $c^{\prime }=c$. The answer is
either yes or no. In the latter case, we receive a counterexample $x\in
(c^{\prime }\setminus c)\cup (c\setminus c^{\prime })$, which is chosen
nondeterministically. The concept $c^{\prime }$ can be considered as a
hypothesis. This hypothesis and corresponding equivalence query are called
proper if $c^{\prime }\in B$.

To translate the results obtained in \cite{Moshkov01,MTCS,MDAM} into exact
learning, we introduce the notion of an infinite family of concepts consisting of an infinite set of elements $U$ and an infinite set $%
C$ of subsets of $U$ called concepts. Each nonempty  finite subset $X$ of the
set $U$ describes a problem of exact learning with the domain $X$ and the
concept class $\{c\cap X:c\in C\}$.  We study the following  modification of this problem: for a
given (but hidden) concept $c\in C$, we should recognize the concept $c\cap X$ using the same queries as for the initial problem. Both the problem and its modification have the same sets of decision trees solving them.

In the present paper, for each infinite family of concepts, we study five functions
of Shannon type, which characterize the dependence in the worst case of the
minimum depth of a decision tree solving an exact learning problem on the
number of elements in the problem description. The considered five functions
correspond to the following five cases:

\begin{enumerate}
\item Only membership queries are used in decision trees.

\item Only equivalence queries are used in decision trees.

\item Both membership and equivalence queries are used in decision trees.

\item Only proper equivalence queries are used in decision trees.

\item Both membership and proper equivalence queries are used in decision
trees.
\end{enumerate}

As in the case of testing problems, the first function has two possible types
of behavior -- logarithmic and linear, and each of the remaining functions
has three possible types of behavior -- constant, logarithmic, and linear.
We also consider joint behavior of these five functions and distinguish
seven complexity classes of infinite families of concepts.

The translation of the results obtained in \cite{Moshkov01,MTCS,MDAM}  into exact learning will be helpful for researchers in this area: the duality of the two directions is simple (attributes and elements from test theory and rough set theory correspond to elements and concepts from exact learning, respectively) but working out the details takes some effort.

The study of decision trees with hypotheses is interesting not only from the theoretical point of view. Experimental results obtained in \cite{ent,Azad21a,Azad21b,ele,CSP} show that  such decision trees can have less complexity than the  conventional decision trees. These results open up some prospects for using decision trees with hypotheses as a means for knowledge representation.

The rest of the paper is organized as follows. Sections \ref{S2} and \ref{S3} present basic notions and main results. Sections \ref{S4}--\ref{S6} contain proofs, and
Section \ref{S7} -- short conclusions.

\section{Basic Notions \label{S2}}

Let $U$ be a nonempty set and $C$ be a nonempty set of subsets of $U$
called \emph{concepts}. The pair $F=(U,C)$ is called a\emph{\ family of
concepts}. If $U$ and $C$ are infinite sets, then the pair $F=(U,C)$ is
called an \emph{infinite family of concepts}. For each element $u\in U$, we
define a function $u:C\rightarrow \{0,1\}$ with the same name as follows:
for any $c\in C$, $u(c)=1$ if and only if $u\in c$.

A \emph{problem over} $F$ is an arbitrary $n$-tuple $z=(u_{1},\ldots ,u_{n})$,
where $n\in \mathbb{N}$, $\mathbb{N}$ is the set of natural numbers $%
\{1,2,\ldots \}$, and $u_{1},\ldots ,u_{n}\in U$. The problem $z$  is as follows: for a given concept $c\in C$, we should recognize  the tuple $z(c)=(u_{1}(c),\ldots
,u_{n}(c))$. The number $\dim z=n$ is called the \emph{%
dimension} of the problem $z$. Denote $U(z)=\{u_{1},\ldots ,u_{n}\}$ and $%
C(z)=\{c\cap U(z):c\in C\}$. Note that the problem $z$ can also be interpreted
as the exact learning problem with the domain $U(z)$ and the concept class $%
C(z)$. We denote by $P(F)$ the set of problems over the family of concepts $F$.

A \emph{system of equations over }$F$ is an arbitrary equation system of the
kind
\[
\{v_{1}(x)=\delta _{1},\ldots ,v_{m}(x)=\delta _{m}\},
\]%
where $m\in \mathbb{N}\cup \{0\}$, $v_{1},\ldots ,v_{m}\in U$, and $\delta
_{1},\ldots ,\delta _{m}\in \{0,1\}$ (if $m=0$, then the considered equation
system is empty). This equation system is called a \emph{system of equations
over }$z$ if $v_{1},\ldots ,v_{m}\in U(z)$. The considered equation system
is called \emph{consistent} (on $C$) if its set of solutions on $C$ is
nonempty. The set of solutions of the empty equation system coincides with $%
C $.

As algorithms for problem $z$ solving, we consider decision trees with two
types of \emph{queries }-- membership and equivalence. We can choose an
element $u_{i}\in U(z)$ and ask about value of the function $u_{i}$. This
\emph{membership} query has two possible answers $\{u_{i}(x)=0\}$ and $%
\{u_{i}(x)=1\}$. We can formulate a \emph{hypothesis over} $z$ in the form $%
H=\{u_{1}(x)=\delta _{1},\ldots ,u_{n}(x)=\delta _{n}\}$, where $\delta
_{1},\ldots ,\delta _{n}\in \{0,1\}$, and ask about this hypothesis. This
\emph{equivalence} query has $n+1$ possible answers: $H,\{u_{1}(x)=\lnot
\delta _{1}\},...,\{u_{n}(x)=\lnot \delta _{n}\}$, where $\lnot 1=0$ and $%
\lnot 0=1$. The first answer means that the hypothesis is true. Other
answers are counterexamples. This hypothesis and the corresponding
equivalence query are called \emph{proper }(for\emph{\ }$F$) if the system
of equations $H$ is consistent on $C$.

A \emph{decision tree over} $z$ is a marked finite directed tree with the
root in which
\begin{itemize}
\item Each \emph{terminal} node is labeled with an $n$-tuple from the set $%
\{0,1\}^{n}$.
\item Each node, which is not terminal (such nodes are called \emph{working}%
), is labeled with an element from the set $U(z)$ (a membership query) or
with a hypothesis over $z$ (an equivalence query).
\item If a working node is labeled with an element $u_{i}$ from $U(z)$, then
there are two edges, which leave this node and are labeled with the systems
of equations $\{u_{i}(x)=0\}$ and $\{u_{i}(x)=1\}$, respectively.
\item If a working node is labeled with a hypothesis
\[
H=\{u_{1}(x)=\delta _{1},\ldots ,u_{n}(x)=\delta _{n}\}
\]%
over $z$, then there are $n+1$ edges, which leave this node and are labeled
with the systems of equations $H,\{u_{1}(x)=\lnot \delta
_{1}\},...,\{u_{n}(x)=\lnot \delta _{n}\}$, respectively.
\end{itemize}
Let $\Gamma $ be a decision tree over $z$. A \emph{complete path} in $\Gamma
$ is an arbitrary directed path from the root to a terminal node in $\Gamma $%
. We now define an equation system $\mathcal{S}(\xi )$ over $F$ associated
with the complete path $\xi $. If there are no working nodes in $\xi $, then
$\mathcal{S}(\xi )\ $is the empty system. Otherwise, $\mathcal{S}(\xi )$ is
the union of equation systems assigned to the edges of the path $\xi $. We
denote by $C(\xi )$ the set of solutions on $C$ of the system of
equations $\mathcal{S}(\xi )$ (if this system is empty, then its solution
set is equal to $C$).

We will say that a decision tree $\Gamma $ over $z$ \emph{solves\ the
problem }$z$\emph{\ relative to} $F$ if, for each concept $c\in C$ and for
each complete path $\xi $ in $\Gamma $ such that $c\in C(\xi )$,
the terminal node of the path $\xi $ is labeled with the tuple $z(c)$.

We now consider an equivalent definition of a decision tree solving a
problem. Denote by $\Delta _{F}(z)$ the set of tuples $(\delta _{1},\ldots
,\delta _{n})\in \{0,1\}^{n}$ such that the system of equations $%
\{u_{1}(x)=\delta _{1},\ldots ,u_{n}(x)=\delta _{n}\}$ is consistent. The
set $\Delta _{F}(z)$ is the set of all possible solutions to the problem $z$%
. 
Let $\Delta
\subseteq \Delta _{F}(z)$, $u_{i_{1}},\ldots ,u_{i_{m}}\in \{u_{1},\ldots
,u_{n}\}$, and $\sigma _{1},\ldots ,\sigma _{m}\in \{0,1\}$. Denote
\[
\Delta (u_{i_{1}},\sigma _{1})\cdots (u_{i_{m}},\sigma _{m})
\]%
the set of all $n$-tuples $(\delta _{1},\ldots ,\delta _{n})\in \Delta $ for
which $\delta _{i_{1}}=\sigma _{1},\ldots ,\delta _{i_{m}}=\sigma _{m}$.

Let $\Gamma $ be a decision tree over the problem $z$. We correspond to each
complete path $\xi $ in the tree $\Gamma $ a word $\pi (\xi )$ in the
alphabet $\{(u_{i},\delta ):u_{i}\in U(z),\delta \in \{0,1\}\}$. If the
equation system $\mathcal{S}(\xi )$ is empty, then $\pi (\xi )$ is the empty
word. If $\mathcal{S}(\xi )=\{u_{i_{1}}(x)=\sigma _{1},\ldots
,u_{i_{m}}(x)=\sigma _{m}\}$, then $\pi (\xi )=(u_{i_{1}},\sigma _{1})\cdots
(u_{i_{m}},\sigma _{m})$. The decision tree $\Gamma $ over $z$ solves the
problem $z$ relative to $F$ if, for each complete path $\xi $ in $\Gamma $,
the set $\Delta _{F}(z)\pi (\xi )$ contains at most one tuple and if this
set contains exactly one tuple, then the considered tuple is assigned to the
terminal node of the path $\xi $.

As time complexity of a decision tree, we consider its \emph{depth} that is
the maximum number of working nodes in a complete path in the tree or, which
is the same, the maximum length of a complete path in the tree. We denote by
$h(\Gamma )$ the depth of a decision tree $\Gamma $.

Let $z\in P(F)$. We denote by $h_{F}^{(1)}(z)$ the minimum depth of a
decision tree over $z$, which solves $z$ relative to $F$ and uses only
membership queries. We denote by $h_{F}^{(2)}(z)$ the minimum depth of a
decision tree over $z$, which solves $z$ relative to $F$ and uses only
equivalence queries. We denote by $h_{F}^{(3)}(z)$ the minimum depth of a
decision tree over $z$, which solves $z$ relative to $F$ and uses both
membership and equivalence queries. We denote by $h_{F}^{(4)}(z)$ the
minimum depth of a decision tree over $z$, which solves $z$ relative to $F$
and uses only proper equivalence queries. We denote by $h_{F}^{(5)}(z)$ the
minimum depth of a decision tree over $z$, which solves $z$ relative to $F$
and uses both membership and proper equivalence queries.

For $i=1,\ldots,5$, we define a function of Shannon type $h_{F}^{(i)}(n)$
that characterizes dependence of $h_{F}^{(i)}(z)$ on $\dim z$ in the worst
case. Let $i\in \{1,\ldots,5\}$ and $n\in \mathbb{N}$. Then%
\[
h_{F}^{(i)}(n)=\max \{h_{F}^{(i)}(z):z\in P(F),\dim z\leq n\}.
\]

\section{Main Results \label{S3}}

Let $F=(U,C)$ be an infinite family of concepts and $r\in \mathbb{N}$. We
will say that the family of concepts $F$ is $r$-\emph{reduced} if, for each
consistent on $C$ system of equations over $F$, there exists a subsystem of
this system that has the same set of solutions on $C$ and contains at most $%
r $ equations. We denote by $\mathcal{R}$ the set of infinite families of
concepts each of which is $r$-reduced for some $r\in \mathbb{N}$.

The next theorem follows from the results obtained in \cite{Moshkov89}, where we
studied closed classes of test tables (decision tables). It also follows
from the results obtained in \cite{Moshkov05}, where we investigated the
weighted depth of decision trees for testing problems.

\begin{theorem}
\label{T1} Let $F$ be an infinite family of concepts. Then the following
statements hold:

(a) If $F\in \mathcal{R}$, then $h_{F}^{(1)}(n)=\Theta (\log n)$.

(b) If $F\notin \mathcal{R}$, then $h_{F}^{(1)}(n)=n$ for any $n\in \mathbb{N%
}$.
\end{theorem}

Let $F=(U,C)$ be an infinite family of concepts. A subset $\{u_{1},\ldots
,u_{m}\}$ of $U$ is \emph{shattered by }$C$ if, for any $\delta _{1},\ldots
,\delta _{m}\in \{0,1\}$, the system of equations $\{u_{1}(x)=\delta
_{1},\ldots ,u_{m}(x)=\delta _{m}\}$ is consistent on the set $C$. This
definition is equivalent to the standard one: a subset $\{u_{1},\ldots
,u_{m}\}$ of $U$ is shattered by $C$ if, for any subset $X$ of $%
\{u_{1},\ldots ,u_{m}\}$, there exists a concept $c\in C$ such that $%
X=\{u_{1},\ldots ,u_{m}\}\cap c$. The empty set of elements is shattered by $%
C$ by definition. We now define the parameter $\mathit{VC}(F)$, which is called the
\emph{Vapnik-Chervonenkis dimension} or \emph{VC-dimension} of the concept
family $F$ \cite{VC}. If, for each $m\in \mathbb{N}$, the set $U$ contains a
subset shattered by $C$ of the cardinality $m$, then $\mathit{VC}(F)=\infty $.
Otherwise, $\mathit{VC}(F)$ is the maximum cardinality of a subset of the set $U$
shattered by $C$. We denote by $\mathcal{D}$ the set of infinite families of
concepts with finite VC-dimension.

Let $F=(U,C)$ be a family of concepts, which is not necessary infinite, $%
u\in U$, and $\delta \in \{0,1\}$. Denote
\[
C(u,\delta )=\{c:c\in C,u(c)=\delta \}.
\]%
We now define inductively the notion of $k$-\emph{family of concepts}, $k\in
\mathbb{N}\cup \{0\}$. The family of concepts $F$ is called $0$-family of
concepts if all functions corresponding to elements  from $U$ are
constant on the set $C$.
Let, for some $k\in \mathbb{N}\cup \{0\}$, the notion of $m$-family of
concepts be defined for $m=0,\ldots ,k$. The family of concepts $F$ is
called $(k+1)$-family of concepts if it is not $m$-family of concepts for $%
m=0,\ldots ,k$ and, for any $u\in U$, there exist numbers $\delta \in
\{0,1\} $ and $m\in \{0,\ldots ,k\}$ such that the family of concepts $%
(U,C(u,\delta ))$ is $m$-family of concepts. It is easy to show by induction
on $k$ that if $F=(U,C)$ is $k$-family of concepts, then $F^{\prime
}=(U^{\prime },C^{\prime })$, where $U^{\prime }\subseteq U$ and $C^{\prime }=\{c\cap U^{\prime }: c \in C \}$, is $l$-family of concepts for
some $l\leq k$. We denote by $\mathcal{C}$ the set of infinite families of
concepts for each of which there exists $k\in \mathbb{N}$ such that the
considered family is $k$-family of concepts.

\begin{theorem}
\label{T2} Let $F$ be an infinite family of concepts. Then the following
statements hold:

(a) If $F\in \mathcal{C}$, then $h_{F}^{(2)}(n)=O(1)$ and $%
h_{F}^{(3)}(n)=O(1)$.

(b) If $F\in \mathcal{D}\setminus \mathcal{C}$, then $h_{F}^{(2)}(n)=\Theta
(\log n)$, $h_{F}^{(3)}(n)=\Omega (\frac{\log n}{\log \log n})$, and $%
h_{F}^{(3)}(n)=O(\log n)$.

(c) If $F\notin \mathcal{D}$, then $h_{F}^{(2)}(n)=n$ and $h_{F}^{(3)}(n)=n$
for any $n\in \mathbb{N}$.
\end{theorem}

From Lemma \ref{P2} below it follows that $\mathcal{C}\subseteq
\mathcal{D}$. Therefore, for any infinite family of concepts $F$, either $%
F\in \mathcal{C}$, or $F\in \mathcal{D}\setminus \mathcal{C}$, or $F\notin
\mathcal{D}$.

Let $F=(U,C)$ be an infinite family of concepts and $r\in \mathbb{N}$. We
will say that the family of concepts $F$ is $r$\emph{-i-reduced} if, for
each inconsistent on $C$ system of equations over $F$, there exists a
subsystem of this system that is inconsistent and contains at most $r$
equations. We denote by $\mathcal{I}$ the set of infinite families of
concepts each of which is $r$-i-reduced for some $r\in \mathbb{N}$.

Since $\mathcal{C}\subseteq
\mathcal{D}$,  for any infinite family of concepts $F$, either $%
F\in \mathcal{C\cap I}$, or $F\in (\mathcal{D}\setminus \mathcal{C)\cap I}$,
or $F\in \mathcal{D}\setminus \mathcal{I}$, or $F\notin \mathcal{D}$.

\begin{theorem}
\label{T3}Let $F$ be an infinite family of concepts. Then the following
statements hold:

(a) If $F\in \mathcal{C\cap I}$, then $h_{F}^{(4)}(n)=O(1)$ and $%
h_{F}^{(5)}(n)=O(1)$.

(b) If $F\in (\mathcal{D}\setminus \mathcal{C)\cap I}$, then $%
h_{F}^{(4)}(n)=\Theta (\log n)$, $h_{F}^{(5)}(n)=\Omega (\frac{\log n}{\log
\log n})$, and $h_{F}^{(5)}(n)=O(\log n)$.

(c) If $F\in \mathcal{D}\setminus \mathcal{I}$ and $i\in \{4,5\}$, then $%
h_{F}^{(i)}(n)\geq n-1$ for infinitely many $n\in \mathbb{N}$ and $%
h_{F}^{(i)}(n)\leq n$ for any $n\in \mathbb{N}$.

(d) If $F\notin \mathcal{D}$, then $h_{F}^{(4)}(n)=n$ and $h_{F}^{(5)}(n)=n$
for any $n\in \mathbb{N}$.
\end{theorem}

Let $F$ be an infinite family of concepts. We now consider the joint
behavior of the functions $h_{F}^{(1)}(n), \ldots ,h_{F}^{(5)}(n)$. It depends on the
belonging of the family of concepts $F$ to the sets $\mathcal{R}$, $\mathcal{%
D}$, $\mathcal{C}$, and $\mathcal{I}$. We correspond to the family of
concepts $F$ its \emph{indicator vector} $ind(U)=(e_{1},e_{2},e_{3},e_{4})%
\in \{0,1\}^{4}$ in which $e_{1}=1$ if and only if $F\in \mathcal{R}$, $%
e_{2}=1$ if and only if $F\in \mathcal{D}$, $e_{3}=1$ if and only if $F\in
\mathcal{C}$, and $e_{4}=1$ if and only if $F\in \mathcal{I}$.

\begin{table}[h]
\caption{Possible indicator vectors of infinite families of concepts}
\label{tab1}\center
\begin{tabular}{|l|llll|}
\hline
& $\mathcal{R}$ & $\mathcal{D}$ & $\mathcal{C}$ & $\mathcal{I}$ \\ \hline
1 & $0$ & $0$ & $0$ & $0$ \\
2 & $0$ & $0$ & $0$ & $1$ \\
3 & $0$ & $1$ & $0$ & $0$ \\
4 & $0$ & $1$ & $0$ & $1$ \\
5 & $0$ & $1$ & $1$ & $0$ \\
6 & $0$ & $1$ & $1$ & $1$ \\
7 & $1$ & $1$ & $0$ & $1$ \\ \hline
\end{tabular}%
\end{table}

\begin{theorem}
\label{T4} For any infinite family of concepts, its indicator vector
coincides with one of the rows of Table \ref{tab1}. Each row of Table \ref%
{tab1} is the indicator vector of some infinite family of concepts.
\end{theorem}

\begin{table}[h]
\caption{Summary of Theorems \protect\ref{T1}--\protect\ref{T4}}
\label{tab2}\center
\begin{tabular}{|l|llll|lllll|}
\hline
& $\mathcal{R}$ & $\mathcal{D}$ & $\mathcal{C}$ & $\mathcal{I}$ & $%
h_{F}^{(1)}(n)$ & $h_{F}^{(2)}(n)$ & $h_{F}^{(3)}(n)$ & $h_{F}^{(4)}(n)$ & $%
h_{F}^{(5)}(n)$ \\ \hline
$\mathcal{F}_{1}$ & $0$ & $0$ & $0$ & $0$ & $n$ & $n$ & $n$ & $n$ & $n$ \\
$\mathcal{F}_{2}$ & $0$ & $0$ & $0$ & $1$ & $n$ & $n$ & $n$ & $n$ & $n$ \\
$\mathcal{F}_{3}$ & $0$ & $1$ & $0$ & $0$ & $n$ & $\Theta (\log n)$ & $%
\approx \log n$ & $\approx n$ & $\approx n$ \\
$\mathcal{F}_{4}$ & $0$ & $1$ & $0$ & $1$ & $n$ & $\Theta (\log n)$ & $%
\approx \log n$ & $\Theta (\log n)$ & $\approx \log n$ \\
$\mathcal{F}_{5}$ & $0$ & $1$ & $1$ & $0$ & $n$ & $O(1)$ & $O(1)$ & $\approx
n$ & $\approx n$ \\
$\mathcal{F}_{6}$ & $0$ & $1$ & $1$ & $1$ & $n$ & $O(1)$ & $O(1)$ & $O(1)$ &
$O(1)$ \\
$\mathcal{F}_{7}$ & $1$ & $1$ & $0$ & $1$ & $\Theta (\log n)$ & $\Theta
(\log n)$ & $\approx \log n$ & $\Theta (\log n)$ & $\approx \log n$ \\ \hline
\end{tabular}%
\end{table}

For $i=1,\ldots ,7$, we denote by $\mathcal{F}_{i}$ the class of all
infinite families of concepts, which indicator vector coincides with the $i$%
th row of Table \ref{tab1}. Table \ref{tab2} summarizes Theorems \ref{T1}--%
\ref{T4}. The first column contains the name of \emph{complexity class} $%
\mathcal{F}_{i}$. The next four columns describe the indicator vector of
families of concepts from this class. The last five columns $h_{F}^{(1)}(n)$%
, ..., $h_{F}^{(5)}(n)$ contain information about behavior of the functions $%
h_{F}^{(1)}(n)$, ..., $h_{F}^{(5)}(n)$ for information systems from the
class $\mathcal{F}_{i}$. The notation $\approx \log n$ in a column $%
h_{F}^{(i)}(n)$ means that $h_{F}^{(i)}(n)$ $=\Omega (\frac{\log n}{\log
\log n})$ and $h_{F}^{(i)}(n)$ $=O(\log n)$. The notation $\approx n$ in a
column $h_{F}^{(i)}(n)$ means that $h_{F}^{(i)}(n)$ $\leq n$ for any $n\in
\mathbb{N}$ and $h_{F}^{(i)}(n)$ $\geq n-1$ for infinitely many $n\in
\mathbb{N}$.

Note that it is possible to consider the union $\mathcal{F}_{1}\cup \mathcal{%
F}_{2}$ of the complexity classes $\mathcal{F}_{1}$ and $\mathcal{F}_{2}$ as
one complexity class since the functions $h_{F}^{(1)}(n),\ldots
,h_{F}^{(5)}(n)$ have the same behavior for information systems from these
classes. However, in this paper, we study $\mathcal{F}_{1}$ and $\mathcal{F}%
_{2}$ as different complexity classes.

\section{Proofs of Theorems \protect\ref{T1} and \protect\ref{T2} \label{S4}}

In this section, we prove Theorems \ref{T1} and \ref{T2}.

\begin{proof}[Proof of Theorem \protect\ref{T1}]
(a) Let $F=(U,C)\in \mathcal{R}$. First, we prove that $h_{F}^{(1)}(n)=O(\log n)$.
Since $F\in \mathcal{R}$, there exists a natural $r$ such that, for each
consistent on $C$ system of equations over $F$, there exists a subsystem of
this system, which has the same set of solutions on $C$ and contains at most
$r$ equations.

Let $z=(u_{1},\ldots ,u_{m})$ be a problem over $F$. The number of equation
systems over $z$ containing at most $r$ equations is at most the number of $%
r $-tuples of equations of the kind $u_{i}=\delta $, where $u_{i}\in
\{u_{1},\ldots ,u_{m}\}$ and $\delta \in \{0,1\}$. The latter number is
equal to $(2m)^{r}$. Therefore $|\Delta _{F}(z)|\leq (2m)^{r}$.

We consider a decision tree $\Gamma $, which solves the problem $z$ relative
to $F$ and uses only membership queries. This tree is constructed by a
halving algorithm that is similar to proposed in \cite{Moshkov83}. We will
describe the work of $\Gamma $ for an arbitrary concept $c$ from $C$. This
work consists of the following steps.

Denote $\Delta =\Delta _{U}(z)$. If $|\Delta |=1$, then the only $n$-tuple
from $\Delta $ is the solution $z(c)$ of the problem $z$ for the concept $c$%
. Let $|\Delta |\geq 2$. For $i=1,\ldots ,m$, we denote by $\delta _{i}$ a
number from $\{0,1\}$ such that $|\Delta (u_{i},\delta _{i})|\geq |\Delta
(u_{i},\lnot \delta _{i})|$. Let $k$ be the maximum number from $\{1,\ldots
,m\}$ for which the equation system $\{u_{1}(x)=\delta _{1},\ldots
,u_{k}(x)=\delta _{k}\}$ is consistent on $C$. This system contains a
subsystem $\{u_{i_{1}}(x)=\delta _{i_{1}},\ldots ,u_{i_{t}}(x)=\delta
_{i_{t}}\}$, which has the same set of solutions on $C$ and for which $t\leq
r$. We now show that $|\Delta (u_{i_{1}},\delta _{i_{1}})\cdots
(u_{i_{t}},\delta _{i_{t}})|\leq |\Delta |/2$. If $k=m$, then $|\Delta
(u_{i_{1}},\delta _{i_{1}})\cdots (u_{i_{t}},\delta _{i_{t}})|\leq 1\leq
|\Delta |/2$. Let $k<m$. Then the equation system $\{u_{i_{1}}(x)=\delta
_{i_{1}},\ldots ,u_{i_{t}}(x)=\delta _{i_{t}},u_{k+1}(x)=\delta _{k+1}\}$ is
inconsistent. Therefore $\Delta (u_{i_{1}},\delta _{i_{1}})\cdots
(u_{i_{t}},\delta _{i_{t}})\subseteq \Delta (u_{k+1},\lnot \delta _{k+1})$
and $|\Delta (u_{i_{1}},\delta _{i_{1}})\cdots (u_{i_{t}},\delta
_{i_{t}})|\leq |\Delta |/2$. We sequentially compute values of the functions
$u_{i_{1}},\ldots ,u_{i_{t}}$ for the concept $c$. If $u_{i_{1}}(c)=\delta
_{i_{1}},\ldots ,u_{i_{t}}(c)=\delta _{i_{t}}$, then during the next step we
will work with the set of tuples $\Delta ^{\prime }=\Delta (u_{i_{1}},\delta
_{i_{1}})\cdots (u_{i_{t}},\delta _{i_{t}})$ for which $|\Delta ^{\prime
}|\leq |\Delta |/2$. If, for some $p\in \{1,\ldots ,t-1\}$, $%
u_{i_{1}}(c)=\delta _{i_{1}},\ldots ,u_{i_{p}}(c)=\delta _{i_{p}}$ and $%
u_{i_{p+1}}(c)=\lnot \delta _{i_{p+1}}$, then during the next step we will
work with the set of tuples
\[
\Delta ^{\prime \prime }=\Delta (u_{i_{1}},\delta _{i_{1}})\cdots
(u_{i_{p}},\delta _{i_{p}})(u_{i_{p+1}},\lnot \delta _{i_{p+1}}).
\]%
It is clear that $|\Delta ^{\prime \prime }|\leq |\Delta |/2$. At this step,
we compute values of at most $r$ functions (make at most $r$ membership
queries) and reduce the number of $m$-tuples (possible solutions) by half.

Let during the work with the concept $c$, the decision tree $\Gamma $ take $%
q $ steps. After $(q-1)$th step, the number of remaining $m$-tuples will be
at least two and at most $(2m)^{r}/2^{q-1}$. Therefore $2^{q}\leq (2m)^{r}$
and $q\leq r\log _{2}(2m)$. So during the processing of the concept $c$, the
decision tree $\Gamma $ computes values of at most $r^{2}\log _{2}(2m)$
functions (makes at most $r^{2}\log _{2}(2m)$ membership queries). Since $c$
is an arbitrary concept from $C$, the depth of $\Gamma $ is at most $%
r^{2}\log _{2}(2m)$. Since $z$ is an arbitrary problem over $F$, we obtain $%
h_{F}^{(1)}(n)=O(\log n)$.

We now show that $h_{F}^{(1)}(n)=\Omega (\log n)$. We prove by induction on $%
n$ that, for any natural $n$, there is a problem $z_{n}=(u_{1},\ldots
,u_{n}) $ over $F$ such that $|\Delta _{F}(z_{n})|\geq n+1$. Since $C$ is an
infinite set of concepts, there exist two concepts $c_{1},c_{2}\in C$ and an
element $u_{1}\in U$ such that $u_{1}\in c_{1}$ and $u_{1}\notin c_{2}$.
Therefore $|\Delta _{F}(z_{1})|\geq 2$, where $z_{1}=(u_{1})$. Let, for some
natural $n$, there exist a problem $z_{n}=(u_{1},\ldots ,u_{n})$ over $F$
such that $|\Delta _{F}(z_{n})|\geq n+1$. Since $C$ is an infinite set of
concepts, there exist two concepts $c_{3},c_{4}\in C$ and an element $%
u_{n+1}\in U$ such that $\{u_{1},\ldots ,u_{n}\}\cap c_{3}=\{u_{1},\ldots
,u_{n}\}\cap c_{4}$, $u_{n+1}\in c_{3}$, and $u_{n+1}\notin c_{4}$. One can
show that $|\Delta _{F}(z_{n+1})|\geq n+2$, where $z_{n+1}=(u_{1},\ldots
,u_{n},u_{n+1})$. Let $n\in \mathbb{N}$, $z_{n}=(u_{1},\ldots ,u_{n})$ be a
problem over $F$ such that $|\Delta _{F}(z_{n})|\geq n+1$, and $\Gamma $ be
a decision tree solving the problem $z$ relative to $F$. Then $\Gamma $
should have at least $n+1$ terminal nodes. One can show that the number of
terminal nodes in the tree $\Gamma $ is at most $2^{h(\Gamma )}$. Therefore $%
n+1\leq 2^{h(\Gamma )}$, $h(\Gamma )\geq \log _{2}(n+1)$, and $%
h_{F}^{(1)}(z)\geq $ $\log _{2}(n+1).$ Thus, $h_{F}^{(1)}(n)=\Omega (\log n)$
and $h_{F}^{(1)}(n)=\Theta (\log n)$.

(b) Let $F=(U,C)\notin \mathcal{R}$. One can show that, for any $n\in \mathbb{N}$,
there is a consistent system of equations $S=\{u_{1}(x)=\delta _{1},\ldots
,u_{n}(x)=\delta _{n}\}$ over $F$ for which each proper subsystem has
different solution set than $S$. Let $\Gamma $ be a decision tree solving
the problem $z=(u_{1},\ldots ,u_{n})$ relative to $F$. Then there is a
complete path $\xi $ in $\Gamma $ such that $\Delta _{F}(z_{n})\pi (\xi
)=\{(\delta _{1},\ldots ,\delta _{n})\}$. From here it follows that $S=%
\mathcal{S}(\xi )$. Therefore the complete path $\xi $ contains at least $n$
working nodes and $h(\Gamma )\geq n.$ Taking into account that $\Gamma $ is
an arbitrary decision tree solving $z$ relative to $F$, we obtain $%
h_{F}^{(1)}(z)\geq n$ and $h_{F}^{(1)}(n)\geq n$. It is easy to show that,
for each problem $z$ over $U$, the inequality $h_{F}^{(1)}(z)\leq \dim z$
holds: to solve the problem $z$, it is enough to compute values of functions
corresponding to all elements from the set $U(z)$. Thus, $h_{F}^{(1)}(n)=n$
for any $n\in \mathbb{N}$.
\end{proof}

We precede the proof of Theorem \ref{T2} by two lemmas.

Let $d\in \mathbb{N}$. A $d$-\emph{complete tree over the family of concepts}
$F=(U,C)$ is a marked finite directed tree with the root in which

\begin{itemize}
\item Each terminal node is not labeled.

\item Each nonterminal node is labeled with an element $u\in U$. There are
two edges leaving this node that are labeled with the systems of equations $%
\{u(x)=0\}$ and $\{u(x)=1\}$, respectively.

\item The length of each complete path (path from the root to a terminal
node) is equal to $d$.

\item For each complete path $\xi $, the equation system $\mathcal{S}(\xi )$%
, which is the union of equation systems assigned to the edges of the path $%
\xi $, is consistent.
\end{itemize}

Let $G$ be a $d$-complete tree over $F$ and $U(G)$ be the set of all
elements attached to the nonterminal nodes of the tree $G$. The number of
nonterminal nodes in $G$ is equal to $2^{0}+2^{1}+\ldots +2^{d-1}=2^{d}-1$.
Therefore $|U(G)|\leq 2^{d}$.

The results mentioned in the following lemma are obtained by methods similar
to used by Littlestone \cite{Littlestone88}, Maass and Tur{\'{a}}n \cite%
{Maass92}, and Angluin \cite{Angluin04}.

\begin{lemma}
\label{L0a}Let $F=(U,C)$ be a family of concepts, $d\in \mathbb{N}$, $G$ be
a $d$-complete tree over $F$, and $z$ be a problem over $U$ such that $%
U(G)\subseteq U(z)$. Then

(a) $h_{F}^{(2)}(z)\geq d$.

(b) $h_{F}^{(3)}(z)\geq \frac{d}{\log _{2}(2d)}$.
\end{lemma}

\begin{proof}
(a) We prove the inequality $h_{F}^{(2)}(z)\geq d$ by induction on $d$. Let $%
d=1$. Then the tree $G$ has the only one nonterminal node, which is labeled
with an element $u$ such that the function $u$ is not constant on $C$.
Therefore $|\Delta _{F}(z)|\geq 2$ and $h_{F}^{(2)}(z)\geq 1$. Let, for $%
t\in \mathbb{N}$ and for any natural $d$, $1\leq d\leq t$, the considered
statement hold. Assume now that $d=t+1$, $G$ is a $d$-complete tree over $F$%
, $z\ $is a problem over $F$ such that $U(G)\subseteq U(z)$, and $\Gamma $
is a decision tree over $z$ with the minimum depth, which solves the problem
$z$ and uses only equivalence queries. Let $u$ be the element attached to the
root of the tree $G$ and $H$ be the hypothesis attached to the root of the
decision tree $\Gamma $. Then there is an edge, which leaves the root of $%
\Gamma $ and is labeled with the equation system $\{u(x)=\delta \}$, where
the equation $u(x)=\lnot \delta $ belongs to the hypothesis $H$. This edge
enters to the root of the subtree of $\Gamma $, which will be denoted by $%
\Gamma _{u}$. There is an edge, which leaves the root of $G$ and is labeled
with the equation system $\{u(x)=\delta \}$. This edge enters to the root of
the subtree of $G$, which will be denoted by $G_{\delta }$. One can show
that the decision tree $\Gamma _{u}$ solves the problem $z$ relative to the
family of concepts $F^{\prime }=(U,C(u,\delta ))$ and $G_{\delta }$ is a $t$%
-complete tree over $F^{\prime }$. It is clear that $U(G_{\delta })\subseteq
U(z)$. Using the inductive hypothesis, we obtain $h(\Gamma _{u})\geq t$.
Therefore $h(\Gamma )\geq t+1=d$ and $h_{F}^{(2)}(z)\geq d$.

(b) We now prove the inequality $h_{F}^{(3)}(z)\geq \frac{d}{\log _{2}(2d)}$%
. Let $z=(u_{1},\ldots ,u)$ and $\Gamma $ be a decision tree over $z$ with
the minimum depth, which solves the problem $z$ and uses both membership and
equivalence queries. The $d$-complete tree $G$ has $2^{d}$ complete paths $%
\xi _{1},\ldots ,\xi _{2^{d}}$. For $i=1,\ldots ,2^{d}$, we denote by $c_{i}$
a solution of the equation system $\mathcal{S}(\xi _{i})$. Denote $%
B=\{c_{1},\ldots ,c_{2^{d}}\}$. We now show that the decision tree $\Gamma $
contains a complete path, which length is at least $\frac{d}{\log _{2}(2d)}$%
. We describe the process of this path construction beginning with the root
of $\Gamma $.

Let the root of $\Gamma $ be labeled with an element $u_{i_{0}}$. For $%
\delta \in \{0,1\}$, we denote by $B^{\delta }$ the set of solutions on $B$
of the equation system $\{u_{i_{0}}(x)=\delta \}$ and choose $\sigma \in
\{0,1\}$ for which $|B^{\sigma }|=\max \{|B^{0}|,|B^{1}|\}$. It is clear
that $|B^{\sigma }|\geq \frac{|B|}{2}\geq \frac{|B|}{2d}$. In the considered
case, the beginning of the constructed path in $\Gamma $ is the root of $%
\Gamma $, the edge that leaves the root and is labeled with the equation
system $\{u_{i_{0}}(x)=\sigma \}$, and the node to which this edge enters.

Let as assume now that the root of $\Gamma $ is labeled with a hypothesis $%
H=\{u_{1}(x)=\delta _{1},\ldots ,u_{n}(x)=\delta _{n}\}$. We denote by $\xi
_{H}$ the complete path in $G$ for which the system of equations $\mathcal{S}%
(\xi _{H})$ is a subsystem of $H$. Let the nonterminal nodes of the complete
path $\xi _{H}$ be labeled with the elements $u_{i_{1}},\ldots ,u_{i_{d}}$.
For $j=1,\ldots ,d$, we denote by $B_{j}$ the set of solutions on $B$ of the
equation system $\{u_{i_{j}}(x)=\lnot \delta _{i_{j}}\}$. It is clear that $%
|B_{1}|+\cdots +|B_{d}|\geq |B|-1$. Therefore there exists $l\in \{1,\ldots
,d\}$ such that $|B_{l}|\geq \frac{|B|-1}{d}\geq \frac{|B|}{2d}$. In the
considered case, the beginning of the constructed path in $\Gamma $ is the
root of $\Gamma $, the edge that leaves the root and is labeled with the
equation system $\{u_{i_{l}}(x)=\lnot \delta _{i_{l}}\}$, and the node to
which this edge enters.

We continue the construction of the complete path in $\Gamma$ in the same
way such that after the $t$th query we will have at least $\frac{|B|}{%
(2d)^{t}}$ elements from $B$. The process of path construction will continue
at least until $\frac{|B|}{(2d)^{t}}\leq 1$, i.e., at least until $\log
_{2}|B|\leq t\log _{2}(2d)$. Since $|B|=2^{d},$ we have $h(\Gamma )\geq t
\geq \frac{d}{\log _{2}(2d)}$ and $h_{F}^{(3)}(z)\geq \frac{d}{\log _{2}(2d)}
$.
\end{proof}

\begin{lemma}
\label{L0b}Let $F=(U,C)$ be a family of concepts, $k\in \mathbb{N}\cup \{0\}$%
, and $F$ be not $m$-family of concepts for $m=0,\ldots ,k$. Then there
exists a $(k+1)$-complete tree over $F$.
\end{lemma}

\begin{proof}
We prove the considered statement by induction on $k$. Let $k=0$. In this
case, $F$ is not $0$-family of concepts. Then there exists an element $u\in
U $ for which the function $u$ is not constant on $C$. Using this element,
it is easy to construct $1$-complete tree over $F$.

Let the considered statement hold for some $k$, $k\geq 0$. We now show that
it also holds for $k+1$. Let $F=(U,C)$ be a family of concepts, which is not
$m$-family of concepts for $m=1,\ldots ,k+1$. Then there exists an element $%
u\in U$ such that, for any $\delta \in \{0,1\}$, the information system $%
F_{\delta }=(U,C(u,\delta ))$ is not $m$-information system for $m=1,\ldots
,k$. Using the inductive hypothesis, we conclude that, for any $\delta \in
\{0,1\}$, there exists a $(k+1)$-complete tree $G_{\delta }$ over $F_{\delta
}$. Denote by $G$ a directed tree with root in which the root is labeled
with the element $u$ and, for any $\delta \in \{0,1\}$, there is an edge
that leaves the root, is labeled with the equation system $\{u(x)=\delta \}$%
, and enters the root of the tree $G_{\delta }$. One can show that the tree $%
G$ is a $(k+2)$-complete tree over $F$.
\end{proof}

\begin{proof}[Proof of Theorem \protect\ref{T2}]
It is clear that $h_{F}^{(3)}(z)\leq h_{F}^{(2)}(z)$ for any problem $z$
over $F$. Therefore $h_{F}^{(3)}(n)\leq h_{F}^{(2)}(n)$ for any $n\in
\mathbb{N}$.

(a) Let $k\in \mathbb{N}\cup \{0\}$. We now show by induction on $k$ that,
for each $k$-family of concepts $F$ (not necessary infinite) for each
problem $z$ over $F$, the inequality $h_{F}^{(2)}(z)\leq k$ holds. Let $%
F=(U,C)$ be a $0$-family of concepts and $z$ be a problem over $F$. Since
all functions corresponding to elements from $U(z)$ are constant on $C$, the
set $\Delta _{F}(z)$ contains only one tuple. Therefore the decision tree
containing only one node labeled with this tuple solves the problem $z$
relative to $F$, and $h_{F}^{(2)}(z)=0$.

Let $k\in \mathbb{N}\cup \{0\}$ and, for each $m$, $0\leq m\leq k$, the
considered statement hold. Let us show that it holds for $k+1$. Let $F=(U,C)$
be a $(k+1)$-family of concepts and $z=(u_{1},\ldots ,u_{n})$ be a problem
over $F$. For $i=1,\ldots ,n$, choose a number $\delta _{i}\in \{0,1\}$ such
that the family of concepts $(U,C(u_{i},\lnot \delta _{i}))$ is $m_{i}$%
-family of concepts, where $1\leq m_{i}\leq k$. Using the inductive
hypothesis, we conclude that, for $i=1,\ldots ,n$, there is a decision tree $%
\Gamma _{i}$ over $z$, which uses only equivalence queries, solves the
problem $z$ over $(U,C(u_{i},\lnot \delta _{i}))$, and has depth at most $%
m_{i}$. We denote by $\Gamma $ a decision tree in which the root is labeled
with the hypothesis $H=\{u_{1}(x)=\delta _{1},\ldots ,u_{n}(x)=\delta _{n}\}$%
, the edge leaving the root and labeled with $H$ enters the terminal node
labeled with the tuple $(\delta _{1},\ldots ,\delta _{n})$, and for $%
i=1,\ldots ,n$, the edge leaving the root and labeled with $\{u_{i}(x)=\lnot
\delta _{i}\}$ enters the root of the tree $\Gamma _{i}$. One can show that $%
\Gamma $ solves the problem $z$ relative to $F$ and $h(\Gamma )\leq k+1$.
Therefore, $h_{F}^{(2)}(z)\leq k+1$ for any problem $z$ over $F$.

Let $F\in \mathcal{C}$. Then $F$ is $k$-family of concepts for some natural $%
k$ and, for each problem $z$ over $F$, we have $h_{F}^{(3)}(z)\leq
h_{F}^{(2)}(z)\leq k$. Therefore $h_{F}^{(2)}(n)=O(1)$ and $%
h_{F}^{(3)}(n)=O(1)$.

(b) Let $F=(U,C)\in \mathcal{D}\setminus \mathcal{C}$. First, we show that $%
h_{F}^{(2)}(n)=O(\log n)$. Let $z=(u_{1},\ldots ,u_{n})$ be an arbitrary
problem over $F$. From Lemma 5.1 \cite{Moshkov05} it follows that $|\Delta
_{F}(z)|\leq (4n)^{\mathit{VC}(F)}$. The proof of this lemma is based on the results
similar to ones obtained by Sauer \cite{Sauer72} and Shelah \cite{Shelah72}.
We consider a decision tree $\Gamma $ over $z$, which solves $z$ relative to
$F$ and uses only equivalence queries. This tree is constructed by the
halving algorithm \cite{Angluin88,Littlestone88}. We describe the work of
this tree for an arbitrary concept $c$ from $C$. Set $\Delta =$ $\Delta
_{F}(z)$. If $|\Delta |=1$, then the only $n$-tuple from $\Delta $ is the
solution $z(c)$ of the problem $z$ for the concept $c$. Let $|\Delta |\geq 2$%
. For $i=1,\ldots ,m$, we denote by $\delta _{i}$ a number from $\{0,1\}$
such that $|\Delta (u_{i},\delta _{i})|\geq |\Delta (u_{i},\lnot \delta
_{i})|$. The root of $\Gamma $ is labeled with the hypothesis $%
H=\{u_{1}(x)=\delta _{1},\ldots ,u_{n}(x)=\delta _{n}\}$. After this query
either the problem $z$ is solved (if the answer is $H$) or we halve the
number of tuples in the set $\Delta $ (if the answer is a counterexample $%
\{u_{i}(x)=\lnot \delta _{i}\}$). In the latter case, set $\Delta =$ $\Delta
_{F}(z)(u_{i},\lnot \delta _{i})$. The decision tree $\Gamma $ continues to
work with the concept $c$ and the set of $n$-tuples $\Delta $ in the same
way. Let during the work with the concept $c$, the considered decision tree
make $q$ queries. After the $(q-1)$th query, the number of remaining $n$%
-tuples in the set $\Delta $ is at least two and at most $%
(4n)^{\mathit{VC}(F)}/2^{q-1}$. Therefore $2^{q}\leq (4n)^{\mathit{VC}(F)}$ and $q\leq
\mathit{VC}(F)\log _{2}(4n)$. So during the processing of the concept $c$, the
decision tree $\Gamma $ makes at most $\mathit{VC}(F)\log _{2}(4n)$ queries. Since $c$
is an arbitrary concept from $C$, the depth of $\Gamma $ is at most $%
\mathit{VC}(F)\log _{2}(4n)$. Since $z$ is an arbitrary problem over $F$, we obtain $%
h_{F}^{(2)}(n)=O(\log n)$. Therefore $h_{F}^{(3)}(n)=O(\log n)$.

Using Lemma \ref{L0b} and the relation $F\notin \mathcal{C}$, we obtain
that, for any $d\in \mathbb{N}$, there exists $d$-complete tree $G_{d}$ over
$F$. Let $U(G_{d})=\{u_{1},\ldots ,u_{n_{d}}\}$. We know that $n_{d}\leq
2^{d}$. Denote $z_{d}=(u_{1},\ldots ,u_{n_{d}})$. From Lemma \ref{L0a} it
follows that $h_{F}^{(2)}(z_{d})\geq d$ and $h_{F}^{(3)}(z_{d})\geq \frac{d}{%
\log _{2}(2d)}$. As a result, we have $h_{F}^{(2)}(2^{d})\geq d$ and $%
h_{F}^{(3)}(2^{d})\geq \frac{d}{\log _{2}(2d)}$. Let $n\in \mathbb{N}$ and $%
n\geq 8$. Then there exists $d\in \mathbb{N}$ such that $2^{d}\leq n<2^{d+1}$%
. We have $d>\log _{2}n-1$, $h_{F}^{(2)}(n)\geq \log _{2}n-1$, $%
h_{F}^{(2)}(n)=\Omega (\log n)$, and $h_{F}^{(2)}(n)=\Theta (\log n)$. It is
easy to show that the function $\frac{x}{\log _{2}(2x)}$ is nondecreasing
for $x\geq 2$. Therefore $h_{F}^{(3)}(n)\geq \frac{\log _{2}n-1}{\log
_{2}(2(\log _{2}n-1))}$ and $h_{F}^{(3)}(n)=\Omega (\frac{\log n}{\log \log n%
})$.

(c) Let $F=(U,C)\notin \mathcal{D}$. We now consider an arbitrary problem $%
z=(u_{1},\ldots ,u_{n})$ over $F$ and a decision tree over $z$, which uses
only equivalence queries and solves the problem $z$ over $F$ in the
following way. For a given concept $c\in C$, the first query is about the
hypothesis $H_{1}=\{u_{1}(x)=1,\ldots ,u_{n}(x)=1\}$. If the answer is $%
H_{1} $, then the problem $z$ is solved for the concept $c$. If, for some $%
i\in \{1,\ldots ,n\}$, the answer is $\{u_{i}(x)=0\}$, then the second query
is about the hypothesis $H_{2}$ obtained from $H_{1}$ by replacing the
equality $u_{i}(x)=1$ with the equality $u_{i}(x)=0$, etc. It is clear that
after at most $n$ equivalence queries the problem $z$ for the concept $c$
will be solved. Thus, $h_{F}^{(2)}(z)\leq n$ and $h_{F}^{(3)}(z)\leq n$.
Since $z$ is an arbitrary problem over $F$, we have $h_{F}^{(2)}(n)\leq n$
and $h_{F}^{(3)}(n)\leq n$ for any $n\in \mathbb{N}$.

Let $n\in \mathbb{N}$. Since $F\notin \mathcal{D}$, there exist elements $%
u_{1},\ldots ,u_{n}\in U$ such that, for any $(\delta _{1},\ldots ,\delta
_{n})\in \{0,1\}^{n}$, the equation system $\{u_{1}(x)=\delta _{1},\ldots
,u_{n}(x)=\delta _{n}\}$ is consistent on $C$. We now consider the problem $%
z=(u_{1},\ldots ,u_{n})$ and an arbitrary decision tree $\Gamma $ over $z$,
which solves the problem $z$ over $F$ and uses both membership and
equivalence queries. Let us show that $h(\Gamma )\geq n$. If $n=1$, then the
considered inequality holds since $|\Delta _{F}(z)|\geq 2$. Let $n\geq 2$.
It is easy to show that an equation system over $z$ is inconsistent if and
only if it contains equations $u_{i}(x)=0$ and $u_{i}(x)=1$ for some $i\in
\{1,\ldots ,n\}$. For each node $v$ of the decision tree $\Gamma $, we
denote by $S_{v}$ the union of systems of equations attached to edges in the
path from the root of $\Gamma $ to $v$. A node $v$ of $\Gamma $ will be
called consistent if the equation system $S_{v}$ is consistent.

We now construct a complete path $\xi $ in the decision tree $\Gamma $,
which nodes are consistent. We will start constructing the path from the
root that is a consistent node. Let the path reach a consistent node $v$
of $\Gamma $. If $v$ is a terminal node, then the path $\xi $ is
constructed. Let $v$ be a working node labeled with an element $u_{i}\in
U(z) $. Then there exists $\delta \in \{0,1\}$ for which the system of
equations $S_{v}\cup \{u_{i}(x)=\delta \}$ is consistent. Then the path $\xi
$ will pass through the edge leaving $v$ and labeled with the system of
equations $\{u_{i}(x)=\delta \}$. Let $v$ be labeled with a hypothesis $%
H=\{u_{1}(x)=\delta _{1},\ldots ,u_{n}(x)=\delta _{n}\}$. If there exists $%
i\in \{1,\ldots ,n\}$ such that the system of equations $S_{v}\cup
\{u_{i}(x)=\lnot \delta \}$ is consistent, then the path $\xi $ will pass
through the edge leaving $v$ and labeled with the system of equations $%
\{u_{i}(x)=\lnot \delta \}$. Otherwise, $S_{v}=H$ and the path $\xi $ will
pass through the edge leaving $v$ and labeled with the system of equations $%
H $.

Let all edges in the path $\xi $ be labeled with systems of equations
containing one equation each. Since all nodes of $\xi $ are consistent, the
equation system $\mathcal{S}(\xi )$ is consistent. We now show that $%
\mathcal{S}(\xi )$ contains at least $n$ equations. Let us assume that this
system contains less than $n$ equations. Then the set $\Delta _{F}(z)\pi
(\xi )$ contains more than one $n$-tuple, which is impossible. Therefore the
length of the path $\xi $ is at least $n$. Let there be edges in $\xi $,
which are labeled with hypotheses, and the first edge in $\xi $ labeled with
a hypothesis $H$ leaves the node $v$. Then $S_{v}=H$ and the length of $\xi $
is at least $n$. Therefore $h(\Gamma )\geq n$, $h_{F}^{(3)}(z)\geq n$, and $%
h_{F}^{(2)}(z)\geq n$. As a result, we obtain $h_{F}^{(3)}(n)\geq n$ and $%
h_{F}^{(2)}(n)\geq n$. Thus, $h_{F}^{(2)}(n)=n$ and $h_{F}^{(3)}(n)=n$ for
any $n\in \mathbb{N}$.
\end{proof}

\section{Proof of Theorem \protect\ref{T3} \label{S5}}

In this section, we prove Theorem \ref{T3}. First, we consider several
auxiliary statements.

\begin{lemma}
\label{L2a} Let $F=(U,C)$ be a family of concepts, $z$ be a problem over $F$%
, and $\Gamma _{1}$ be a decision tree over $z$ that solves\ the problem $z$%
\ relative to $F$ and uses both membership and proper equivalence queries.
Then there exists a decision tree $\Gamma _{2}$ over $z$ that solves\ the
problem $z$\ relative to $F$, uses only proper equivalence queries, and
satisfies the inequality $h(\Gamma _{2})\leq 2^{h(\Gamma _{1})}-1$.
\end{lemma}

\begin{proof}
We prove this statement by the induction on the depth $h(\Gamma _{1})$ of
the decision tree $\Gamma _{1}$. Let $h(\Gamma _{1})=0$. Then, as the
decision tree $\Gamma _{2}$, we can take the decision tree $\Gamma _{1}$. It
is clear that $h(\Gamma _{2})=2^{h(\Gamma _{1})}-1$. We now assume that the
considered statement is true for any family of concepts, any problem over
this family, and any decision tree over the considered problem that solves\
this problem, uses both membership and proper equivalence queries, and has
depth at most $k$, $k\geq 0$.

Let $F=(U,C)$ be a family of concepts, $z=(u_{1},\ldots ,u_{n})$ be a
problem over $F$, and $\Gamma _{1}$ be a decision tree over $z$ that solves\
the problem $z$\ relative to $F$, uses both membership and proper
equivalence queries, and satisfies the condition $h(\Gamma _{1})=k+1$. We
now show that there exists a decision tree $\Gamma _{2}$ over $z$, which
solves\ the problem $z$\ relative to $F$, uses only proper equivalence
queries, and which depth is at most $2^{k+1}-1$.

Let the root of $\Gamma _{1}$ be labeled with a proper hypothesis $%
H=\{u_{1}(x)=\delta _{1},\ldots ,u_{n}(x)=\delta _{n}\}$. Then there are $%
n+1 $ edges, which leave the root, are labeled with the systems of equations
$H$, $\{u_{1}(x)=\lnot \delta _{1}\}$, ..., $\{u_{n}(x)=\lnot \delta _{n}\}$%
, and enter the roots of subtrees $G_{0},G_{1},\ldots ,G_{n}$ of the tree $%
\Gamma _{1}$, respectively. It is clear that, for $i=1,\ldots ,n$, $G_{i}$
is a decision tree over $z$, which solves the problem $z$ relative to the
family of concepts $F_{i}=(U,C(u_{i},\lnot \delta _{i}))$, uses only
membership queries and proper equivalence queries for $F_{i}$, and satisfies
the inequality $h(G_{i})\leq k$. Using the inductive hypothesis, we obtain
that, for $i=1,\ldots ,n$, there exists a decision tree $G_{i}^{\prime }$
over $z$ that solves\ the problem $z$\ relative to $F_{i}$, uses only proper
equivalence queries for $F_{i}$, and satisfies the inequalities $%
h(G_{i}^{\prime })\leq 2^{h(G_{i})}-1\leq 2^{k}-1$. Let $G_{0}^{\prime }$ be
the decision tree, which contains only one node labeled with the tuple $%
(\delta _{1},\ldots ,\delta _{n})$. We denote by $\Gamma _{2}$ the decision
tree over $z$ that is obtained from the decision tree $\Gamma _{1}$ by
replacing the subtrees $G_{0},G_{1},\ldots ,G_{n}$ with the subtrees $%
G_{0}^{\prime },G_{1}^{\prime },\ldots ,G_{n}^{\prime }$. It is easy to show
that $\Gamma _{2}$ is a decision tree over $z$, which solves the problem $z$
relative to $F$, uses only proper equivalence queries for $F$, and satisfies
the inequalities $h(\Gamma _{2})\leq 2^{k}-1+1\leq 2^{h(\Gamma _{1})}-1$.

Let the root of $\Gamma _{1}$ be labeled with an element $u_{i}$. Then there
are two edges, which leave the root, are labeled with the equation systems $%
\{u_{i}(x)=0\}$ and $\{u_{i}(x)=1\}$, and enter the roots of subtrees $T_{0}$
and $T_{1}$ of the tree $\Gamma _{1}$, respectively. It is clear that, for $%
p=0,1$, $T_{p}$ is a decision tree over $z$, which solves the problem $z$
relative to the family of concepts $F_{p}=(U,C(u_{i},p))$, uses only
membership queries and proper equivalence queries for $F_{p}$, and satisfies
the inequality $h(T_{p})\leq k$. Using the inductive hypothesis, we obtain
that, for $p=0,1$, there exists a decision tree $T_{p}^{\prime }$ over $z$
that solves\ the problem $z$\ relative to $F_{p}$, uses only proper
equivalence queries for $U_{p}$, and satisfies the inequalities $%
h(T_{p}^{\prime })\leq 2^{h(T_{p})}-1\leq 2^{k}-1$. We denote by $T$ the
decision tree obtained from the decision tree $T_{0}^{\prime }$ by replacing
each terminal node of $T_{0}^{\prime }$ with the decision tree $%
T_{1}^{\prime }$.

Denote by $\Gamma _{2}$ the decision tree obtained from $T$ by the following
transformation of each complete path $\xi $ in $T$. If $C(\xi
)=\emptyset $, then we keep the path $\xi $ untouched. Let $C(\xi
)\neq \emptyset $, $\bar{\delta}=(\delta _{1},\ldots ,\delta _{n})$ be the
tuple that was attached to the terminal node of the tree $T_{0}^{\prime }$
through which the path $\xi $ passes, and $\bar{\sigma}=(\sigma _{1},\ldots
,\sigma _{n})$ be the tuple attached to the terminal node of $\xi $. Since $%
C(\xi )\neq \emptyset $, at least one of the tuples $\bar{\delta}$
and $\bar{\sigma}$ belongs to the set $\Delta _{F}(z)$. Let, for the
definiteness, $\bar{\delta}\in \Delta _{F}(z)$. Denote $H=\{u_{1}(x)=\delta
_{1},\ldots ,u_{n}(x)=\delta _{n}\}$. We replace the terminal node of the
path $\xi $ with the working node labeled with the hypothesis $H$, which is
proper for $F$. There are $n+1$ edges that leave this node and are labeled
with the systems of equations $H,\{u_{1}(x)=\lnot \delta
_{1}\},...,\{u_{n}(x)=\lnot \delta _{n}\}$, respectively. The edge labeled
with $H$ enters to the terminal node labeled with the tuple $\bar{\delta}$.
All other edges enter to terminal nodes labeled with the tuple $\bar{\sigma}$%
. One can show that $\Gamma _{2}$ is a decision tree over $z$ that solves\
the problem $z$\ relative to $F$, uses only proper equivalence queries for $%
U $, and satisfies the relations $h(\Gamma _{2})\leq
2(2^{k}-1)+1=2^{h(\Gamma _{1})}-1$.
\end{proof}

\begin{lemma}
\label{L3a} Let $F=(U,C)$ be an infinite family of concepts. Then $%
h_{F}^{(3)}(n)\leq h_{F}^{(5)}(n)\leq h_{F}^{(4)}(n)\leq n$ and $%
h_{F}^{(2)}(n)\leq h_{F}^{(4)}(n)$ for any $n\in \mathbb{N}$.
\end{lemma}

\begin{proof}
It is clear, that $h_{F}^{(3)}(z)\leq h_{F}^{(5)}(z)\leq h_{F}^{(4)}(z)$ and
$h_{F}^{(2)}(z)\leq h_{F}^{(4)}(z)$ for any problem $z$ over $U$. Therefore $%
h_{F}^{(3)}(n)\leq h_{F}^{(5)}(n)\leq h_{F}^{(4)}(n)$ and $%
h_{F}^{(2)}(n)\leq h_{F}^{(4)}(n)$ for any $n\in \mathbb{N}$.

We now consider an arbitrary problem $z=(u_{1},\ldots ,u_{n})$ over $F$ and
a decision tree over $z$, which uses only proper equivalence queries for $F$
and solves the problem $z$ relative to $F$ in the following way. For a given
concept $c\in C$, the first query is about an arbitrary proper hypothesis $%
H_{1}=\{u_{1}(x)=\delta _{1},\ldots ,u_{n}(x)=\delta _{n}\}$ for $F$. If the
answer is $H_{1}$, then the problem $z$ is solved for the concept $c$. If,
for some $i\in \{1,\ldots ,n\}$, the answer is $\{u_{i}(x)=\lnot \delta
_{i}\}$, then the second query is about a proper hypothesis $%
H_{2}=\{u_{1}(x)=\sigma _{1},\ldots ,u_{n}(x)=\sigma _{n}\}$ such that $%
\sigma _{i}=\lnot \delta _{i}$. If the answer is $H_{2}$, then the problem $%
z $ is solved for the concept $c$. If, for some $j\in \{1,\ldots ,n\}$, the
answer is $\{u_{j}(x)=\lnot \sigma _{j}\}$, then the third query is about a
proper hypothesis $H_{3}=\{u_{1}(x)=\gamma _{1},\ldots ,u_{n}(x)=\gamma
_{n}\}$ such that $\gamma _{i}=\lnot \delta _{i}$ and $\gamma _{j}=\lnot
\sigma _{j}$, etc. It is clear that after at most $n$ queries the problem $z$
for the concept $c$ will be solved. Thus, $h_{F}^{(4)}(z)\leq n$. Since $z$
is an arbitrary problem over $F$, we have $h_{F}^{(4)}(n)\leq n$ for any $%
n\in \mathbb{N}$.
\end{proof}

\begin{proof}[Proof of Theorem \protect\ref{T3}]
(a) Let $r\in \mathbb{N}$. We now show by induction on $k\in \mathbb{N}\cup
\{0\}$ that, for each $r$-i-reduced $k$-family of concepts $F$ (not
necessary infinite) for each problem $z$ over $F$, the inequality $%
h_{F}^{(5)}(z)\leq rk$ holds.

Let $F=(U,C)$ be a $r$-i-reduced $0$-family of concepts and $z$ be a problem
over $F$. Since all functions corresponding to elements from $U(z)$ are
constant on $C$, the set $\Delta _{F}(z)$ contains only one tuple. Therefore
the decision tree consisting of one node labeled with this tuple solves the
problem $z$ relative to $F$, and $h_{F}^{(5)}(z)=0$.

Let $k\in \mathbb{N}\cup \{0\}$ and, for each $m$, $0\leq m\leq k$, the
considered statement hold. Let us show that it holds for $k+1$. Let $F=(U,C)$
be a $r$-i-reduced $(k+1)$-family of concepts and $z=(u_{1},\ldots ,u_{n})$
be a problem over $F$. For $i=1,\ldots ,n$, choose a number $\delta _{i}\in
\{0,1\}$ such that the family of concepts $(U,C(u_{i},\lnot \delta _{i}))$
is $m_{i}$-family of concepts, where $1\leq m_{i}\leq k$. It is easy to show
that $(U,C(u_{i},\lnot \delta _{i}))$ is $r$-i-reduced family of concepts.
Using the inductive hypothesis, we conclude that, for $i=1,\ldots ,n$, there
is a decision tree $\Gamma _{i}$ over $z$, which uses both membership
queries and proper equivalence queries for $(U,C(u_{i},\lnot \delta _{i}))$,
solves the problem $z$ relative to $(U,C(u_{i},\lnot \delta _{i}))$, and has
depth at most $rm_{i}$.

Let the hypothesis $H=\{u_{1}(x)=\delta _{1},\ldots ,u_{n}(x)=\delta _{n}\}$
be proper for $F$. We denote by $T_{1}$ a decision tree in which the root is
labeled with the hypothesis $H$, the edge leaving the root and labeled with $%
H$ enters the terminal node labeled with the tuple $(\delta _{1},\ldots
,\delta _{n})$, and for $i=1,\ldots ,n$, the edge leaving the root and
labeled with $\{u_{i}(x)=\lnot \delta _{i}\}$ enters the root of the tree $%
\Gamma _{i}$. One can show that $T_{1}$ is a decision tree over $z$, which
uses both membership queries and proper equivalence queries for $F$, solves
the problem $z$ relative to $F$, and satisfies the inequalities $%
h(T_{1})\leq rk+1\leq r(k+1)$.

Let the hypothesis $H$ be not proper for $F$. Then the equation system $%
\{u_{1}(x)=\delta _{1},\ldots ,u_{n}(x)=\delta _{n}\}$ is inconsistent on $C$%
, and there exists its subsystem $\{u_{i_{1}}(x)=\delta _{i_{1}},\ldots
,u_{i_{t}}(x)=\delta _{i_{t}}\}$, which is inconsistent on $C$ and for which
$t\leq r$. We denote by $G$ a decision tree over $z$ with $2^{t}$ terminal
nodes in which each terminal node is labeled with $n$-tuple $(0,\ldots ,0)$,
and each complete path contains $t$ working nodes labeled with elements $%
u_{i_{1}},\ldots ,u_{i_{t}}$ starting from the root. We denote by $T_{2}$ a
decision tree obtained from the decision tree $G$ by transformation of each
complete path $\xi $ in $G$. Let $\{u_{i_{1}}(x)=\sigma _{1}\},\ldots
,\{u_{i_{t}}(x)=\sigma _{t}\}$ be equation systems attached to edges leaving
the working nodes of $\xi $ labeled with the elements $u_{i_{1}},\ldots
,u_{i_{t}}$, respectively. If $(\sigma _{1},\ldots ,\sigma _{t})=(\delta
_{i_{1}},\ldots ,\delta _{i_{t}})$, then we keep the path $\xi $ untouched.
Otherwise, let $j$ be the minimum number from the set $\{1,\ldots ,t\}$ such
that $\sigma _{j}=\lnot \delta _{i_{j}}$. In this case, we replace the
terminal node of the path $\xi $ with the root of the decision tree $\Gamma
_{i_{j}}$. One can show that $T_{2}$ is a decision tree over $z$, which uses
both membership and proper equivalence queries, solves the problem $z$
relative to $F$, and satisfies the inequalities $h(T_{2})\leq rk+t\leq
r(k+1) $. Therefore, $h_{F}^{(5)}(z)\leq r(k+1)$ for any problem $z$ over $F$%
.

Let $F\in \mathcal{C\cap I}$. Then $F$ is $r$-i-reduced $k$-family of
concepts for some natural $r$ and $k$, and $h_{F}^{(5)}(z)\leq rk$ for each
problem $z$ over $F$. From Lemma \ref{L2a} it follows that $%
h_{F}^{(4)}(z)\leq 2^{rk}-1$ for each problem $z$ over $F$. Therefore $%
h_{F}^{(4)}(n)=O(1)$ and $h_{F}^{(5)}(n)=O(1)$.

(b) Let $F=(U,C)\in (\mathcal{D}\setminus \mathcal{C)\cap I}$. By Lemma \ref%
{L3a}, $h_{F}^{(5)}(n)\geq h_{F}^{(3)}(n)$ and $h_{F}^{(4)}(n)\geq
h_{F}^{(2)}(n)$ for any $n\in \mathbb{N}$. Using the fact that $U\in
\mathcal{D}\setminus \mathcal{C}$ and Theorem \ref{T2}, we obtain $%
h_{F}^{(2)}(n)=\Omega (\log n)$ and $h_{F}^{(3)}(n)=\Omega (\frac{\log n}{%
\log \log n})$. Therefore $h_{F}^{(4)}(n)=\Omega (\log n)$ and $%
h_{F}^{(5)}(n)=\Omega (\frac{\log n}{\log \log n})$.

Since the family of concepts $F$ belongs to the set $\mathcal{D}$, it has
finite VC-dimension $\mathit{VC}(F)$. Since $F\in \mathcal{I}$, the family of
concepts $F$ is $r$-i-reduced for some natural $r$. We assume that $r\geq 2$%
. We can do it because each $t$-i-reduced family of concepts, $t\in \mathbb{N%
}$, is $(t+1)$-i-reduced.

We now show that $h_{F}^{(4)}(n)=O(\log n)$. Let $z=(u_{1},\ldots ,u_{n})$
be an arbitrary problem over $F$. From Lemma 5.1 \cite{Moshkov05} it follows
that $|\Delta _{F}(z)|\leq (4n)^{\mathit{VC}(F)}$.

We consider a decision tree $\Gamma $ over $z$, which solves the problem $z$
relative to $F$ and uses only proper equivalence queries. This tree is
constructed by a variant of the halving algorithm \cite%
{Angluin04,Hegedus95,Hellerstein96}. We describe the work of this tree for
an arbitrary concept $c$ from $C$. Set $\Delta =$ $\Delta _{F}(z)$. If $%
|\Delta |=1$, then the only $n$-tuple from $\Delta $ is the solution $z(c)$
to the problem $z$ for the concept $c$. Let $|\Delta |\geq 2$. For $%
i=1,\ldots ,n$, we denote by $\delta _{i}$ a number from $\{0,1\}$ such that
$|\Delta (u_{i},\delta _{i})|\geq |\Delta (u_{i},\lnot \delta _{i})|$.

Let the system of equations $H=\{u_{1}(x)=\delta _{1},\ldots
,u_{n}(x)=\delta _{n}\}$ be consistent on $C$. In this case, the root of $%
\Gamma $ is labeled with the proper hypothesis $H$. After this query, either
the problem $z$ will be solved (if the answer is $H$) or the number of
remaining tuples in $\Delta $ will be at most $|\Delta |/2$ (if the answer
is a counterexample $\{u_{i}(x)=\lnot \delta _{i}\}$).

Let the system of equations $H$ be inconsistent on $C$. For any inconsistent
subsystem $B$ of $H$, there exists a subsystem $D$ of $B$, which is
inconsistent and contains at most $r$ equations. Then the system $D$
contains at least one equation $u_{i}(x)=\delta _{i}$ such that $|\Delta
(u_{i},\lnot \delta _{i})|\geq |\Delta |/r$. If we assume the contrary, we
obtain that the system $D$ is consistent, which is impossible. Let $u_{i}\in
\{u_{1},\ldots ,u_{n}\}$. The element $u_{i}$ is called balanced if $|\Delta
(u_{i},\lnot \delta _{i})|\geq |\Delta |/r$, and unbalanced if $|\Delta
(u_{i},\lnot \delta _{i})|<|\Delta |/r$.

We denote by $H_{u}$ the subsystem of $H$ consisting of all equations $%
u_{i}(x)=\delta _{i}$ from $H$ with unbalanced elements $u_{i}$. We now show
that the system $H_{u}$ is consistent. Let us assume the contrary. Then it
will contain at least one equation for balanced element, which is
impossible. Let $b$ be a solution from $C$ to the system $H_{u}$, and $%
u_{1}(b)=\sigma _{1},\ldots ,u_{n}(b)=\sigma _{n}$. Then the system of
equations $P=\{u_{1}(x)=\sigma _{1},\ldots ,u_{n}(x)=\sigma _{n}\}$ is
consistent on $C$.

In the considered case, the root of $\Gamma $ is labeled with the proper
hypothesis $P$. After this query, either the problem $z$ will be solved (if
the answer is $P$), or the number of remaining tuples in $\Delta $ will be
less than $|\Delta |/r$ (if the answer is a counterexample $\{u_{i}(x)=\lnot
\sigma _{i}\}$ and $u_{i}$ is an unbalanced element), or the number of
remaining tuples in $\Delta $ will be at most $|\Delta |/2$ (if the answer
is a counterexample $\{u_{i}(x)=\lnot \sigma _{i}\}$, $\sigma _{i}=\delta
_{i}$, and $u_{i}$ is a balanced element), or the number of remaining
tuples in $\Delta $ will be at most $|\Delta |(1-1/r)$ (if the answer is a
counterexample $\{u_{i}(x)=\lnot \sigma _{i}\}$, $\sigma _{i}=\lnot \delta
_{i}$, and $u_{i}$ is a balanced element).

After the first query ($H$ or $P$) of the decision tree $\Gamma $, either
the problem $z$ will be solved or the number of remaining tuples in $\Delta $
will be at most $|\Delta |(1-1/r)$. In the latter case when the answer is a
counterexample of the kind $\{u_{i}(x)=\lnot \gamma _{i}\}$ ($\gamma
_{i}=\delta _{i}$ if the first query is $H$ and $\gamma _{i}=\sigma _{i}$ if
the first query is $P$) set $\Delta =$ $\Delta _{U}(z)(u_{i},\lnot \gamma
_{i})$. It is easy to show that the family of concepts $(U,C(u_{i},\lnot
\gamma _{i}))$ is also $r$-i-reduced. The decision tree $\Gamma $ continues
to work with the concept $c$ and the set of $n$-tuples $\Delta $ in the same
way.

Let during the work with the concept $c$, the decision tree $\Gamma $ make $%
q $ queries. After the $(q-1)$th query, the number of remaining $n$-tuples
in the set $\Delta $ is at least two and at most $(4n)^{\mathit{VC}(F)}(1-1/r)^{q-1}$%
. Therefore $(1+1/(r-1))^{q}\leq (4n)^{\mathit{VC}(F)}$ and $q\ln (1+1/(r-1))\leq
\mathit{VC}(F)\ln (4n)$. Taking into account that $\ln (1+1/m)>1/(m+1)$ for any
natural $m$, we obtain $q\,<r\mathit{VC}(F)\ln (4n)$. So during the processing of the
concept $c$, the decision tree $\Gamma $ makes at most $r\mathit{VC}(F)\ln (4n)$
queries. Since $c$ is an arbitrary concept from $C$, the depth of $\Gamma $
is at most $r\mathit{VC}(F)\ln (4n)$ and $h_{F}^{(4)}(z)\leq r\mathit{VC}(F)\ln (4n)$. Since $%
z $ is an arbitrary problem over $F$, we obtain $h_{F}^{(4)}(n)=O(\log n)$.
By Lemma \ref{L3a}, $h_{F}^{(5)}(n)=O(\log n)$.

(c) Let $F=(U,C)\in \mathcal{D}\setminus \mathcal{I}$. From Lemma \ref{L3a}
it follows that $h_{F}^{(5)}(n)\leq h_{F}^{(4)}(n)\leq n$ for any $n\in
\mathbb{N}$. We now show that, for any $m\in \mathbb{N}$, there exists a
natural $n$ such that $n\geq m$, $h_{F}^{(4)}(n)\geq n-1$, and $%
h_{F}^{(5)}(n)\geq n-1$.

Let $m\in \mathbb{N}$. Since $F\notin \mathcal{I}$, there exists a system of
equations $P$ over $U$ with $n\geq m$ equations such that $P$ is
inconsistent but each proper subsystem of $P$ is consistent on $C$. Let, for
the definiteness, $P=\{u_{1}(x)=0,\ldots ,u_{n}(x)=0\}$. Consider the
problem $z=(u_{1},\ldots ,u_{n})$ over $F$. Then, for $i=1,\ldots ,n$, the
set $\Delta _{F}(z)$ contains $n$-tuple $\bar{\delta}_{i}=(0,\ldots
,0,1,0,\ldots ,0)$ in which all digits with the exception of the $i$th one
are equal to $0$.

Let $\Gamma $ be a decision tree over $z$ that solves the problem $z$
relative to $F$ and uses both membership and proper equivalence queries. We
consider a complete path $\xi $ in $\Gamma $ in which each edge is labeled with an equation system of the kind $\{u_i(x)=0\}$, where $u_i \in U(z)$. Such complete path exists since $P$ is not a proper hypothesis. Let  $\pi (\xi
)=(u_{i_{1}},0)\cdots (u_{i_{t}},0)$ for some $u_{i_{1}},\ldots ,u_{i_{t}}\in U(z)$.  Since $\Gamma $ solves the problem $z$, the set $%
\Delta _{F}(z)\pi (\xi )$ contains at most one tuple. If we assume that $%
t<n-1$, we obtain that $\Delta _{F}(z)\pi (\xi )$ contains at least two
tuples. Therefore $t\geq n-1$ and $h(\Gamma )\geq n-1$. Thus, $%
h_{F}^{(5)}(z)\geq n-1$, $h_{F}^{(5)}(n)\geq n-1$ and, by Lemma \ref{L3a}, $%
h_{F}^{(4)}(n)\geq n-1$.

(d) Let $F\notin \mathcal{D}$. From Lemma \ref{L3a} it follows that $%
h_{F}^{(3)}(n)\leq h_{F}^{(5)}(n)\leq h_{F}^{(4)}(n)\leq n$ for any $n\in
\mathbb{N}$. By Theorem \ref{T2}, $h_{F}^{(3)}(n)=n$ for any $n\in \mathbb{N}$.
Thus, $h_{F}^{(5)}(n)=h_{F}^{(4)}(n)=n$ for any $n\in \mathbb{N}$.
\end{proof}

\section{Proof of Theorem \protect\ref{T4} \label{S6}}

First, we prove several auxiliary statements.

\begin{lemma}
\label{P1} $\mathcal{R}\subseteq \mathcal{D}$.
\end{lemma}

\begin{proof}
Let $F\in \mathcal{R}$. By Theorem \ref{T1}, $h_{F}^{(1)}(n)=\Theta (\log n)$%
. Let us assume that $F\notin \mathcal{D}$. Then, for any $n\in \mathbb{N}$,
there exists a problem $z=(u_{1},\ldots ,u_{n})$ over $F$ such that $|\Delta
_{F}(z)|=2^{n}$. Let $\Gamma $ be a decision tree over $z$, which solves the
problem $z$ relative to $F$ and uses only membership queries. Then $\Gamma $
should have at least $2^{n}$ terminal nodes. One can show that the number of
terminal nodes in the tree $\Gamma $ is at most $2^{h(\Gamma )}$. Then $%
2^{n}\leq 2^{h(\Gamma )}$, $h(\Gamma )\geq n$, and $h_{F}^{(1)}(z)\geq $ $n.$
Therefore $h_{F}^{(1)}(n)\geq n$ for any $n\in \mathbb{N}$, which is
impossible. Thus, $\mathcal{R}\subseteq \mathcal{D}$.
\end{proof}

\begin{lemma}
\label{P2} $\mathcal{C}\subseteq \mathcal{D}$.
\end{lemma}

\begin{proof}
Let $F\in \mathcal{C}$. By Theorem \ref{T2}, $h_{F}^{(2)}(n)=O(1)$. Let us
assume that $F\notin \mathcal{D}$. Then, by Theorem \ref{T2}, $%
h_{F}^{(2)}(n)=n$ for any $n\in \mathbb{N}$, which is impossible. Therefore $%
\mathcal{C}\subseteq \mathcal{D}$.
\end{proof}

\begin{lemma}
\label{P3} $\mathcal{R}\cap \mathcal{C}=\emptyset $.
\end{lemma}

\begin{proof}
Assume the contrary: $\mathcal{R}\cap \mathcal{C}\neq \emptyset $ and $%
F=(U,C)\in \mathcal{R}\cap \mathcal{C}$. Let $r,k\in \mathbb{N}$ and $F$ be $r$%
-reduced $k$-family of concepts. We now
consider an arbitrary problem $z=(u_{1},\ldots ,u_{n})$ over $F$ and
describe a decision tree $\Gamma $ over $z$, which uses only membership
queries, solves the problem $z$ over $F$, and has depth at most $kr$.

For $i=1,\ldots ,n$, let $\delta _{i}$ be a number from $\{0,1\}$ such that $%
(U,C(u_{i},\lnot \delta _{i}))$ is $m_{i}$-family of concepts with $0\leq
m_{i}<k$. Let $t$ be the maximum number from the set $\{1,\ldots ,n\}$ such
that the system of equations $S=\{u_{1}(x)=\delta _{1},\ldots
,u_{t}(x)=\delta _{t}\}$ is consistent. Then there exists a subsystem $%
\{u_{i_{1}}(x)=\delta _{i_{1}},\ldots ,u_{i_{p}}(x)=\delta _{i_{p}}\}$ of
the system $S$, which has the same set of solutions as $S$ and for which $%
p\leq r$. For a given $c\in C$, the decision tree $\Gamma $ computes
sequentially values $u_{i_{1}}(c),\ldots ,u_{i_{p}}(c)$.

If, for some $q\in \{1,\ldots ,p\}$, $u_{i_{1}}(c)=\delta _{i_{1}},\ldots
,u_{i_{q-1}}(c)=\delta _{i_{q-1}}$, and $u_{i_{q}}(c)=\lnot \delta _{i_{q}}$%
, then the decision tree $\Gamma $ continues to work with the problem $z$
and the family of concepts $F^{\prime }=(U,C^{\prime })$, where $C^{\prime }$
is the set of solutions on $C$ of the equation system $\{u_{i_{1}}(x)=\delta
_{i_{1}},\ldots ,u_{i_{q-1}}(x)=\delta _{i_{q-1}},u_{i_{q}}(x)=\lnot \delta
_{i_{q}}\}$. One can show that $F^{\prime }$ is $l^{\prime }$-family of concepts for
some $l^{\prime }\leq m_{i_{q}}<k$.

Let $u_{i_{1}}(c)=\delta _{i_{1}},\ldots ,u_{i_{p}}(c)=\delta _{i_{p}}$. If $t=n$,
then $(\delta _{1},\ldots ,\delta _{n})$ is the solution of the problem $z$
for the considered concept $c$. Let $t<n$. Then the decision tree $\Gamma $
continues to work with the problem $z$ and the family of concepts $F^{\prime
\prime }=(U,C^{\prime \prime })$, where $C^{\prime \prime }$ is the set of
solutions on $C$ of the equation system $\{u_{i_{1}}(x)=\delta
_{i_{1}},\ldots ,u_{i_{p}}(x)=\delta _{i_{p}}\}$. We know that the equation
system $\{u_{1}(x)=\delta _{1},\ldots ,u_{t}(x)=\delta
_{t},u_{t+1}(x)=\delta _{t+1}\}$ is inconsistent. Therefore the system $%
\{u_{i_{1}}(x)=\delta _{i_{1}},\ldots ,u_{i_{p}}(x)=\delta
_{i_{p}},u_{t+1}(x)=\delta _{t+1}\}$ is inconsistent. Hence $C^{\prime
\prime }\subseteq C(u_{t+1},\lnot \delta _{t+1})$ and $F^{\prime \prime }$
is $l^{\prime \prime }$-family of concepts for some $l^{\prime \prime }\leq
m_{t+1}<k$.

As a result, after at most $r$ membership queries, we either solve the
problem $z$ or reduce the consideration of the problem $z$ over $k$-family
of concepts $F$ to the consideration of the problem $z$ over some $l$-family
of concepts, where $l<k$. After at most $rk$ membership queries, we solve the
problem $z$ since each problem over $0$-family of concepts has exactly one
possible solution. Therefore $h_{F}^{(1)}(z)\leq rk$ and $%
h_{F}^{(1)}(n)=O(1) $. By Theorem \ref{T1}, $h_{F}^{(1)}(n)=\Theta (\log n)$%
. The obtained contradiction shows that $\mathcal{R}\cap \mathcal{C}%
=\emptyset $.
\end{proof}

\begin{lemma}
\label{P1a} $\mathcal{R}\subseteq \mathcal{I}$.
\end{lemma}

\begin{proof}
Let $F=(U,C)\in \mathcal{R}$. Then $F$ is $r$-restricted for some natural $r$%
. We now show that $F$ is $(r+1)$-i-restricted. Let $S$ be an arbitrary
inconsistent on $C$ equation system over $F$ and $S^{\prime }$ be a
subsystem of $S$ with the maximum number of equations that is consistent.
Since $F$ is $r$-restricted, the system $S^{\prime }$ has a subsystem $%
S^{\prime \prime }$ with at most $r$ equations and the same set of solutions
on $C$ as the system $S^{\prime }$. It is clear that there exists an
equation $u(x)=\delta $ from $S$ such that the system of equations $%
S^{\prime }\cup \{u(x)=\delta \}$ is inconsistent. Then the subsystem $%
S^{\prime \prime }\cup \{u(x)=\delta \}$ of $S$ with at most $r+1$ equations
is inconsistent. Therefore $F$ is $(r+1)$-i-restricted and $F\in \mathcal{I}$%
.
\end{proof}

\begin{table}[h]
\caption{All $3$-tuples from the set $\{0,1\}^{3}$}
\label{tab3}\center
\begin{tabular}{|l|lll|}
\hline
& $\mathcal{R}$ & $\mathcal{D}$ & $\mathcal{C}$ \\ \hline
1 & $0$ & $0$ & $0$ \\
2 & $0$ & $1$ & $0$ \\
3 & $0$ & $1$ & $1$ \\
4 & $1$ & $1$ & $0$ \\
5 & $1$ & $0$ & $0$ \\
6 & $0$ & $0$ & $1$ \\
7 & $1$ & $0$ & $1$ \\
8 & $1$ & $1$ & $1$ \\ \hline
\end{tabular}%
\end{table}

Let $F$ be an infinite family of concepts and its indicator vector $ind(F)$
be equal to $(e_{1},e_{2},e_{3},e_{4})$. The vector $(e_{1},e_{2},e_{3})$
will be called the restricted indicator vector for the family of concepts $F$
and will be denoted $rind(U)$. In this vector, $e_{1}=1$ if and only if $%
F\in \mathcal{R}$, $e_{2}=1$ if and only if $F\in \mathcal{D}$, and $e_{3}=1$
if and only if $F\in \mathcal{C}$.

\begin{lemma}
\label{P4} For any infinite family of concepts, its restricted indicator
vector coincides with one of the rows of Table \ref{tab3} with numbers 1-4.
\end{lemma}

\begin{proof}
Table \ref{tab3} contains as rows all $3$-tuples from the set $\{0,1\}^{3}$.
We now show that rows with numbers 5-8 cannot be restricted indicator
vectors of infinite families of concepts. Assume the contrary: there is $i\in
\{5,6,7,8\}$ such that the row with the number $i$ is the restricted
indicator vector of an infinite family of concepts $F$. If $i=5$, then $F\in
\mathcal{R}$ and $F\notin \mathcal{D}$, but this is impossible since, by
Lemma \ref{P1}, $\mathcal{R}\subseteq \mathcal{D}$. If $i=6$, then $%
F\in \mathcal{C}$ and $F\notin \mathcal{D}$, but this is impossible since,
by Lemma \ref{P2}, $\mathcal{C}\subseteq \mathcal{D}$. If $i=7$, then $%
F\in \mathcal{R}$ and $F\notin \mathcal{D}$, but this is impossible since,
by Lemma \ref{P1}, $\mathcal{R}\subseteq \mathcal{D}$. If $i=8$, then $%
F\in \mathcal{R}$ and $F\in \mathcal{C}$, but this is impossible since, by
Lemma \ref{P3}, $\mathcal{R}\cap \mathcal{C}=\emptyset $. Therefore,
for any infinite family of concepts, its restricted indicator vector
coincides with one of the rows of Table \ref{tab3} with numbers 1-4.
\end{proof}

\begin{table}[h]
\caption{All extensions of rows 1-4 of Table \protect\ref{tab3}}
\label{tab4}\center
\begin{tabular}{|l|llll|}
\hline
& $\mathcal{R}$ & $\mathcal{D}$ & $\mathcal{C}$ & $\mathcal{I}$ \\ \hline
1 & $0$ & $0$ & $0$ & $0$ \\
2 & $0$ & $0$ & $0$ & $1$ \\
3 & $0$ & $1$ & $0$ & $0$ \\
4 & $0$ & $1$ & $0$ & $1$ \\
5 & $0$ & $1$ & $1$ & $0$ \\
6 & $0$ & $1$ & $1$ & $1$ \\
7 & $1$ & $1$ & $0$ & $1$ \\
8 & $1$ & $1$ & $0$ & $0$ \\ \hline
\end{tabular}%
\end{table}

\begin{lemma}
\label{P2a} For any infinite family of concepts, its indicator vector
coincides with one of the rows of Table \ref{tab1}.
\end{lemma}

\begin{proof}
Let $F$ be an infinite family of concepts and $%
ind(F)=(e_{1},e_{2},e_{3},e_{4})$. Then $rind(F)=(e_{1},e_{2},e_{3})$. By
Lemma \ref{P4}, $(e_{1},e_{2},e_{3})$ is one of the rows of Table \ref{tab3}
with numbers 1-4. Therefore, for each infinite family of concepts, its
indicator vector is an extension of one of the rows of Table \ref{tab3} with
numbers 1-4: it can be obtained from the row by adding the fourth digit,
which is equal to $0$ or $1$. Table \ref{tab4} contains all extensions of
rows of Table \ref{tab3} with numbers 1-4. We now show that the row with
number 8 cannot be the indicator vector of an infinite family of concepts.
Assume the contrary: there is an infinite family of concepts $F^{\prime }$
such that $ind(F^{\prime })=(1,1,0,0)$. Then $F^{\prime }\in \mathcal{R}$
and $F^{\prime }\notin \mathcal{I}$, but this is impossible since, by
Lemma \ref{P1a}, $\mathcal{R}\subseteq \mathcal{I}$. Therefore, for
any infinite family of concepts, its indicator vector coincides with one of
the rows of Table \ref{tab4} with numbers 1-7. Thus, it coincides with one
of the rows of Table \ref{tab1}.
\end{proof}

We now define seven infinite families of concepts $F_{1},\ldots ,F_{7}$ and
prove that these families belong to the complexity classes $\mathcal{F}%
_{1},\ldots ,\mathcal{F}_{7}$, respectively.

Define an infinite family of concepts $F_{1}=(U_{1},C_{1})$ as follows: $%
U_{1}$ is the set of all infinite sequences $u=(u^{(i)})_{i\in \mathbb{N}}$,
where $u^{(i)}\in \{0,1\}$ for any $i\in \mathbb{N}$, $C_{1}=\{c_{i}:i\in
\mathbb{N}\}$ and, for any $u=(u^{(i)})_{i\in \mathbb{N}}\in U_{1}$ and $%
c_{i}\in C_{1}$, $u\in c_{i}$ if and only if $u^{(i)}=1$.

\begin{lemma}
\label{L1b}The family of concepts $F_{1}$ belongs to the class $\mathcal{F}%
_{1}$.
\end{lemma}

\begin{proof}
It is easy to show that the family of concepts $F_{1}$ has infinite
VC-dimension. Therefore $F_{1}\notin \mathcal{D}$. We now show that $%
F_{1}\notin \mathcal{I}$. Let $n\in \mathbb{N}$. We now define elements  $%
u_{0}=(u_{0}^{(i)})_{i\in \mathbb{N}},u_{1}=(u_{1}^{(i)})_{i\in \mathbb{N}%
},\ldots ,u_{n}=(u_{n}^{(i)})_{i\in \mathbb{N}}\in U_{1}$. For any $i\in
\mathbb{N}$, $u_{0}^{(i)}=1$ if and only if $i\in \{1,\ldots ,n\}$. For $%
j=1,\ldots ,n$, $u_{j}^{(i)}=1$ if and only if $j=i$. It is easy to show that
the equation system $\{u_{0}(x)=1,u_{1}(x)=0,\ldots ,u_{n}(x)=0\}$ is
inconsistent on $C_{1}$ but each proper subsystem of this system is
consistent. Therefore $F_{1}\notin \mathcal{I}$. Using Lemma \ref{P2a}, we
obtain $ind(F_{1})=(0,0,0,0)$, i.e., $F_{1}\in \mathcal{F}_{1}$.
\end{proof}

Define an infinite family of concepts $F_{2}=(U_{2},C_{2})$ as follows: $%
U_{2}=\mathbb{N}$ and $C_{2}$ is the set of all subsets of the set $\mathbb{N%
}$.

\begin{lemma}
\label{L2b}The family of concepts $F_{2}$ belongs to the class $\mathcal{F}%
_{2}$.
\end{lemma}

\begin{proof}
It is easy to show that the family of concepts $F_{2}$ has infinite
VC-dimension. Therefore $F_{2}\notin \mathcal{D}$. Let $S$ be a system of
equations over $U_{2}$. It is clear that $S$ is inconsistent if and only if,
for some $i\in \mathbb{N}$, the system $S$ contains equations $i(x)=0$ and $%
i(x)=1$. Therefore $F_{2}$ is $2$-i-restricted and $F_{2}\in \mathcal{I}$.
Using Lemma \ref{P2a}, we obtain $ind(U_{2})=(0,0,0,1)$, i.e., $%
F_{2}\in \mathcal{F}_{2}$.
\end{proof}

Define an infinite family of concepts $F_{3}=(U_{3},C_{3})$ as follows: $%
U_{3}=\{p_{i}:i\in \mathbb{N}\}\cup \{l_{i}:i\in \mathbb{N}\}$
and $C_{3}=\{c_{i}:i\in \mathbb{N}\}$, where $c_{1}=\{p_{1}\}$ and, for $i\geq 2$, $c_{i}=\{p_{i},l_{1},\ldots ,l_{i-1}\}$.

\begin{lemma}
\label{L3b}The family of concepts $F_{3}$ belongs to the class $\mathcal{F}%
_{3}$.
\end{lemma}

\begin{proof}
For $n\in \mathbb{N}$, denote $S_{n}=\{p_{1}(x)=0,\ldots
,p_{n}(x)=0,l_{n}(x)=0\}$. It is easy to show that the equation system $S_{n}
$ is inconsistent on $C_3$ and each proper subsystem of $S_{n}$ is consistent.
Therefore $F_{3}\notin \mathcal{I}$. By Lemma \ref{P1a}, $F_{3}\notin
\mathcal{R}$. Using elements from the set $\{l_{i}:i\in \mathbb{N}\}$, we
can construct $d$-complete tree over $F_{3}$ for each $d\in \mathbb{N}$. By
Lemma \ref{L0a} and Theorem \ref{T2}, $F_{3}\notin \mathcal{C}$. One can
show that $\mathit{VC}(F_{3})=1$. Therefore $F_{3}\in \mathcal{D}$. Thus, $%
ind(U_{3})=(0,1,0,0)$, i.e., $F_{3}\in \mathcal{F}_{3}$.
\end{proof}

Define an infinite binary information system $F_{4}=(U_{4},C_{4})$ as
follows: $U_{4}=\{u_{i}:i\in \mathbb{N}\}\cup \{u_{i,j}:i,j\in \mathbb{N}\}$
and  $C_{4}=\{c_{p,q}:p,q\in \mathbb{N}\}$. For any $u_{i}\in U_{4}$
and any $c_{p,q}\in C_{4}$, $u_{i}\in c_{p,q}$ if and only if $p>i$. For any
$u_{i,j}\in U_{4}$ and any $c_{p,q}\in C_{4}$, $u_{i,j}\in c_{p,q}$ if and
only if $(p,q)=(i,j)$.

\begin{lemma}
\label{L4b}The family of concepts $F_{4}$ belongs to the class $\mathcal{F}%
_{4}$.
\end{lemma}

\begin{proof}
Let $n\in \mathbb{N}$ and $S_{n}=\{u_{1,1}(x)=0,\ldots ,u_{1,n}(x)=0\}$. It
is easy to show that the system $S_{n}$ is consistent and each proper
subsystem of $S_{n}$ has another set of solutions on $C_{4}$ than the system
$S_{n}$. Therefore $F_{4}\notin \mathcal{R}$.

Using elements from the set $\{u_{i}:i\in \mathbb{N}\}$, we can construct $d$%
-complete tree over $F_{4}$ for each $d\in \mathbb{N}$. By Lemma \ref{L0a}
and Theorem \ref{T2}, $F_{4}\notin \mathcal{C}$.

Let $S$ be an equation system over $F_{4}$. One can show that $S$ is
inconsistent if and only if $S$ contains at least one of the following pairs
of equations:
\begin{itemize}
\item $u_{i,j}(x)=0$ and $u_{i,j}(x)=1$;

\item $u_{i,j}(x)=1$ and $u_{k,l}(x)=1$, $(i,j)\neq (k,l)$;

\item $u_{i,j}(x)=1$ and $u_{k}(x)=0$, $i>k$;

\item $u_{i,j}(x)=1$ and $u_{k}(x)=1$, $i\leq k$;

\item $u_{i}(x)=0$ and $u_{j}(x)=1$, $i\leq j$.
\end{itemize}
Therefore $F_{4}$ is $2$-i-restricted and $F_{4}\in \mathcal{I}$. One can
show that $\mathit{VC}(F_{4})=1$. Therefore $F_{4}\in \mathcal{D}$. Thus, $%
ind(F_{4})=(0,1,0,1)$, i.e., $F_{4}\in \mathcal{F}_{4}$.
\end{proof}

Define an infinite family of concepts $F_{5}=(U_{5},C_{5})$ as follows: $$%
U_{5}=\bigcup_{i\in \mathbb{N}}\{u_{i},u_{i,1},\ldots ,u_{i,i}\}$$ and  $%
C_{5}=\bigcup_{i\in \mathbb{N}}\{c_{i,1},\ldots ,c_{i,i}\}$. For any $%
u_{i}\in U_{5}$ and any $c\in C_{5}$, $u_{i}\in c$ if and only if $c\in
\{c_{i,1},\ldots ,c_{i,i}\}$. For any $u_{i,j}\in U_{5}$ and any $c\in C_{5}$%
, $u_{i,j}\in c$ if and only if $c=c_{i,j}$.

\begin{lemma}
\label{L5b}The family of concepts $F_{5}$ belongs to the class $\mathcal{F}%
_{5}$.
\end{lemma}

\begin{proof}
It is easy to show that $F_{5}$ is $2$-family of concepts. In particular, $%
(U_{5},C_{5}(u_{1},1))$ is $0$-family of concepts, $%
(U_{5},C_{5}(u_{i},1))$ is $1$-family of concepts if $i \ge 2$, and $%
(U_{5},C_{5}(u_{i,j},1))$ is $0$-family of concepts for any element $%
u_{i,j}\in U_{5}$. Therefore $F_{5}\in \mathcal{C}$. Let $i\in \mathbb{N}$
and $S_{i}=\{u_{i}(x)=1,u_{i,1}(x)=0,\ldots ,u_{i,i}(x)=0\}$. One can show
that the equation system $S_{i}$ is inconsistent and each proper subsystem of $S_{i}$ is
consistent. Therefore $F_{5}\notin \mathcal{I}$. Using Lemma \ref{P2a}%
, we obtain $ind(F_{5})=(0,1,1,0)$, i.e., $F_{5}\in \mathcal{F}_{5}$.
\end{proof}

Define an infinite family of concepts $F_{6}=(U_{6},C_{6})$ as follows: $%
U_{6}=\{p_{i}:i\in \mathbb{N}\}$ and $C_{6}=\{c_{i}:i\in \mathbb{N}\}$, where $%
c_{i}=\{p_{i}\}$ for any $i\in \mathbb{N}$.

\begin{lemma}
\label{L6b}The family of concepts $F_{6}$ belongs to the class $\mathcal{F}%
_{6}$.
\end{lemma}

\begin{proof}
It is easy to show that $F_{6}$ is $1$-family of concepts: evidently, $F_{6}$
is not $0$-family of concepts, and $(U_{6},C_{6}(p_{i},1))$ is $0$-family of
concepts for any $i\in \mathbb{N}$. Therefore $F_{6}\in \mathcal{C}$. Let $S$
be an equation system over $F_{6}$. One can show that $S$ is inconsistent if
and only if it contains equations $p_{i}(x)=0$ and $p_{i}(x)=1$ for some $%
i\in \mathbb{N}$ or it contains equations $p_{i}(x)=1$ and $p_{j}(x)=1$ for
some $i,j\in \mathbb{N}$, $i\neq j$. Therefore $F_{6}$ is $2$-i-restricted
and $F_{6}\in \mathcal{I}$. Using Lemma \ref{P2a}, we obtain $%
ind(F_{6})=(0,1,1,1)$, i.e., $F_{6}\in \mathcal{F}_{6}$.
\end{proof}

Define an infinite family of concepts $F_{7}=(U_{7},C_{7})$ as follows: $%
U_{7}=\{l_{i}:i\in \mathbb{N}\}$ and $C_{7}=\{c_{i}:i\in \mathbb{N}\}$, where $%
c_{1}=\emptyset $ and, for $i\geq 2$, $c_{i}=\{l_{1},\ldots ,l_{i-1}\}$.

\begin{lemma}
\label{L7b}The family of concepts $F_{7}$ belongs to the class $\mathcal{F}%
_{7}$.
\end{lemma}

\begin{proof}
Let us consider an arbitrary consistent system of equations $S$ over $F_{7}$%
. We now show that there is a subsystem of $S$, which has at most two
equations and the same set of solutions as $S$. Let $S$ contain both
equations of the kind $l_{i}(x)=1$ and $l_{j}(x)=0$. Denote $i_{0}=\max
\{i:l_{i}(x)=1\in S\}$ and $j_{0}=\min \{j:l_{j}(x)=0\in S\}$. One can show
that the system of equations $S^{\prime }=\{l_{i_{0}}(x)=1,l_{j_{0}}(x)=0\}$
has the same set of solutions as $S$. The case when $S$ contains for some $%
\delta \in \{0,1\}$ only equations of the kind $l_{p}(x)=\delta $ can be
considered in a similar way. In this case, the equation system $S^{\prime }$
contains only one equation. Therefore the family of concepts $F_{7}$ is $2$%
-reduced and $F_{7}\in $ $\mathcal{R}$. Using Lemma \ref{P2a}, we
obtain $ind(F_{7})=(1,1,0,1)$, i.e., $F_{7}\in \mathcal{F}_{7}$.
\end{proof}

\begin{proof}[Proof of Theorem \protect\ref{T4}]
From Lemma \ref{P2a} it follows that, for any infinite family of
concepts, its indicator vector coincides with one of the rows of Table \ref%
{tab1}. Using Lemmas \ref{L1b}-\ref{L7b}, we conclude that each row of Table %
\ref{tab1} is the indicator vector of some infinite family of concepts. 
\end{proof}

\section{Conclusions \label{S7}}

Based on the results of exact learning, test theory, and rough set theory, for an arbitrary
infinite family of concepts, we studied five functions, which characterize
the dependence in the worst case of the minimum depth of a decision tree
solving a problem of exact learning on the number of elements in the problem
description. These five functions correspond to (i) decision trees using
membership queries, (ii) decision trees using equivalence queries,  (iii)
decision trees using both membership and equivalence queries, (iv) decision
trees using proper equivalence queries, and (v) decision trees using both
membership and proper equivalence queries. We described possible types of
behavior for each of these five functions. We also studied join behavior of
these functions and distinguished seven complexity classes of
infinite family of concepts. In the future, we are also planing to study
partial equivalence queries \cite{Maass92}.

\subsection*{Acknowledgments}

Research reported in this publication was supported by King Abdullah
University of Science and Technology (KAUST).

\bibliographystyle{spmpsci}
\bibliography{exact-families}

\end{document}